\documentclass[twoside]{article}

\usepackage[accepted]{aistats2023}

\usepackage{star}
\usepackage{hyperref}
\usepackage{mathtools}
\hypersetup{
    colorlinks=true,
    linkcolor=blue,
    filecolor=blue,
    urlcolor=blue,
    citecolor=blue,
}
\usepackage{amsmath}

\DeclareFontFamily{U}{mathx}{}
\DeclareFontShape{U}{mathx}{m}{n}{<-> mathx10}{}
\DeclareSymbolFont{mathx}{U}{mathx}{m}{n}
\DeclareMathAccent{\widecheck}{0}{mathx}{"71}
\usepackage[round]{natbib}
\begin{document}



\twocolumn[

\aistatstitle{Statistical Analysis of Karcher Means for Random Restricted PSD Matrices}

\aistatsauthor{ Hengchao Chen \And Xiang Li \And  Qiang Sun }

\aistatsaddress{ Department of Statistical Sciences\\ University of Toronto \And  School of Mathematical Sciences\\Peking University \And Department of Statistical Sciences\\ University of Toronto } ]

\begin{abstract}
	 Non-asymptotic statistical analysis is often missing for modern geometry-aware machine learning algorithms due to the possibly intricate non-linear manifold structure.
	 This paper studies an intrinsic mean model on the manifold of restricted positive semi-definite matrices and provides a non-asymptotic statistical analysis of the Karcher mean. We also consider a general extrinsic signal-plus-noise model, under which a deterministic error bound of the Karcher mean is provided. As an application, we show that the distributed principal component analysis algorithm, LRC-dPCA, achieves the same performance as the full sample PCA algorithm. Numerical experiments lend strong support to our theories.
\end{abstract}

\section{Introduction}

Positive semi-definite (PSD) matrices arise in a wide range of applications, such as covariance matrices in statistics \citep{Wainwright19}, kernel matrices in machine learning \citep{hastie2009elements}, diffusion tensor images in medical imaging \citep{dryden2009non}, semi-definite programming \citep{journee2010low}, and covariance descriptors in image set classification \citep{wang2012covariance}, to name a few. 
From a geometric perspective, the cone of PSD matrices is not a vector space, since linear combinations of multiple PSD matrices are not necessarily PSD matrices.  
Instead, the set of (restricted) PSD matrices of fixed rank has been endowed with different metrics such that it forms a Riemannian manifold \citep{Bonnabel10,Vandereycken13,Massart20,Neuman21}. By utilizing the geometric structures, researchers have developed many powerful statistical or computational methods \citep{faraki2016image,cornea2017regression,patrangenaru2016nonparametric}.

One important concept in Riemannian geometry or more generally metric spaces is the Karcher mean \citep{Karcher77}. The Karcher mean is often referred to as the Fr$\rm \acute{e}$chet mean or the barycenter of mass. Given $M$ points $\{\bz_m\}_{m=1}^M$ on a metric space $(\cM,d)$ with distance function $d(\cdot,\cdot)$, the Karcher mean $\tilde\bz$ of these points is given by
\#
\tilde\bz=\argmin_{\bz\in\cM}\sum_md^2(\bz,\bz_m).
\#
When the underlying space is Euclidean, the Karcher mean is reducesd to the arithmetic mean.
In general, the existence and computation of the Karcher mean is already complicated due to the possibly intricate non-Euclidean structure~\citep{Karcher77,bini2013computing}.
% , not to mention the statistical properties.
As a result, most works focus on the computation and applications of the Karcher mean, while few provide statistical guarantees. 
% without analyzing its statistical properties.
% Therefore, most works focus on the computation and applications of Karcher mean without statistical theories. 
Statistically,  \cite{bhattacharya2003large,bhattacharya2005large} establish a large sample theory of the Karcher mean on manifolds with applications to spheres and projective spaces. 
\cite{bigot2013minimax} shows the minimax optimality of the Karcher mean of discretely sampled curves. 
In this paper, we consider the manifold of restricted PSD matrices by \cite{Neuman21}. In particular, we first study an intrinsic mean model, inspired by the geometric structure of the restricted PSD manifold. A non-asymptotic statistical analysis of the Karcher mean is provided under this intrinsic model. We further consider a general extrinsic signal-plus-noise model, which does not necessarily coincide with the manifold geometry by \cite{Neuman21}. For this general model, we give a deterministic error bound for the Karcher mean, which is then used to provide an error bound for a distributed principal component analysis algorithm.

The Karcher mean is closely related to distributed learning problems, especially the divide-and-conquer (DC) framework~\citep{Mackey11}. In distributed learning problems, massive datasets are scattered across distant servers and directly fusing these datasets is challenging due to concerns on communication cost, privacy, data security, and ownership, among others. A commonly used distributed framework is the DC framework which first computes local estimators locally and then aggregate them on the central server, where the last step is often equivalent to computing the Karcher mean on certain manifolds. For example, the divide-and-conquer principal component analysis (PCA) algorithms \citep{Fan19,Bhaskara19,Neuman21} essentially compute the Karcher means on the Grasssmann manifold, Euclidean space, or the manifold of restricted PSD matrices, respectively. 
%However, the statistical properties of the Karcher means for random objects on the restricted PSD manifold remain unknown.
Motivated by this observation, we give theoretical guarantees of the DC PCA algorithm, LRC-dPCA, proposed in \cite{Neuman21} by applying our non-asymptotic statistical analysis of the Karcher mean on the restricted PSD manifold.
Specifically, we show that given sufficienly large local sample size, LRC-dPCA achieves the same performance as the full sample PCA algorithm, which outputs the top eigenvectors of the covariance matrix based on full data.

Our contributions are three-fold. First, we provide a non-asymptotic statistical analysis of the Karcher mean on the restricted PSD manifold under an intrinsic model. Second, for a generic signal-plus-noise model, we give a deterministic characterization of the Karcher mean and then obtain a deterministic error bound. Third, as an application, we show that LRC-dPCA and full sample PCA share the same performance given sufficiently large local sample size. Numerical experiments are carried out to support our theories.

The rest of this paper proceeds as follows. We conclude this section with a discussion on related works. Section~\ref{sec:2} reviews the geometry for restricted PSD matrices proposed in \cite{Neuman21}. Then in Section~\ref{sec:3}, we provide the theoretical analysis of the Karcher mean on the restricted PSD manifolds. Applications to distributed PCA algorithms are given in Section~\ref{sec:4}. Numerical experiments are carried out in Section~\ref{sec:5} and we give concluding remarks in Section~\ref{sec:6}. Proofs are left to the Appendix.

\subsection{Related work}
{\bf Manifolds of PSD matrices} The cone of symmetric positive definite (SPD) matrices is not a vector space. It can be viewed as different Riemannian manifolds when endowed with different metrics, such as the affine-invariant metric~\citep{Moakher05} and the Log-Euclidean metric~\citep{Arsigny07}. It is, however, non-trivial to generalize these metrics to the rank-deficient (PSD) case. To this end, \citet{Bonnabel10} treated a PSD matrix of rank $K$ in a quotient space as a $K$-dimensional subspace coupled with a $K$-by-$K$ SPD matrix and then endowed the manifold of PSD matrices with a weighted product metric. Using this geometry, \citet{Bonnabel13} developed a rank-preserving geometric mean of PSD matrices. Later, \citet{Vandereycken13} viewed a PSD manifold as a homogeneous space and \citet{Massart20} analyzed a quotient geometry on the manifold of PSD matrices. However, it is hard to give a statistical model on these manifolds. More recently, \citet{Neuman21} proposed a geometry for restricted PSD matrices which has closed-form solutions for many geometric concepts including the Karcher mean. Our paper provides statistical analysis of the Karcher mean corresponding to this geometry.

{\bf Distributed PCA} 
%There are a bunch of distributed algorithms, including the divide-and-conquer (DC) framework~\cite{Mackey11}, the shift-and-invert algorithm~\cite{Garber16}, and the communication-efficient surrogate likelihood (CSL) framework~\cite{Jordan19}. 
%On the distributed PCA problem, 
To estimate the leading eigenvector, \citet{Garber17} proposed a sign-fixing averaging approach. To estimate the top $K$ eigenspace, \citet{Fan19} proposed a projector averaging approach and \citet{Charisopoulos21} proposed to average local eigenvector matrices after carefully rotating them. Disregarding the information of eigenvalues, both \citet{Fan19}'s and \citet{Charisopoulos21}'s methods require the knowledge of the precise location $K$ of a large eigen gap. To alleviate this issue, \citet{Bhaskara19} proposed to average the local rank-$K$ approximation matrices and then conduct PCA on the aggregated matrix. \citet{Neuman21} utilized the same methods as \citet{Bhaskara19} except that the average of local rank-$K$ approximation matrices is taken on the manifold of restricted PSD matrices. \citet{Neuman21} did not provide statistical analysis for their proposed method, while our paper fixes this gap as an application of the main results. Another branch of research turns PCA into the problem of solving a linear system and then solves distributed PCA by some multi-round algorithms. Among them, some make use of the shift-and-invert framework~\citep{Garber17,Chen21}, while some use incremental update schemes~\citep{gang2019fast,grammenos2020federated,li2021communication}.

%\subsection{Paper organization and notations} 

%We summarize some notations that will be used throughout this paper. 
{\bf Notation.}
By convention, we use regular letters for scalars and bold letters for both vectors and matrices. 
%For an integer $M$, we abbreviate $[M]$ for the set $\{1,\ldots,M\}$. 
Given a vector $\bu\in\RR^p$, denote by $\norm{\bu}_2$ its $\ell_2$ norm. Given a matrix $\bA\in\RR^{n\times p}$, we use $\norm{\bA}_{\rF}$, $\norm{\bA}_2$ and $\norm{\bA}_{\max}=\max_{i,j}|\bA_{ij}|$ to denote its Frobenius norm, $\ell_2$ norm and max norm, respectively. We use $\textnormal{span}(\bA)$ to represent the subspace spanned by the columns of $\bA$. For a symmetric matrix $\bA$, denote by $\lambda_j(\bA)$ its $j$th largest eigenvalue. For two sequences of real numbers $\cbr{a_n}_{n\geq 1}$ and $\cbr{b_n}_{n\geq 1}$, we write $a_n \lesssim b_n$ (or $a_n \gtrsim b_n$) if $a_n \leq C b_n$ (or $a_n \geq C b_n$) for some constant $C>0$ independent of $n$. For an infinitesimal number $\epsilon$, we denote a matrix whose Frobenius norm or max norm is $\cO(\epsilon)$ (i.e., $\lesssim\epsilon$) by $\cO_{\rF}(\epsilon)$ or $\cO_{\max}(\epsilon)$, respectively. Given a random variable $x\in\RR$, we define $\norm{x}_{\psi_2}=\sup_{p\geq 1}(\EE |x|^p)^{1/p}/\sqrt{p}$ and $\norm{x}_{\psi_1}=\sup_{p\geq 1}(\EE |x|^p)^{1/p}/p$. 
%We refer interested readers to \citet{Vershynin12} for equivalent definitions of $\psi_2$ norm and $\psi_1$ norm. 
Given two integers $p\geq K>0$, we denote by $\cO_{p\times K}$ the set of matrices in $\RR^{p\times K}$ whose columns are orthonormal. Denote by $S(p,K)$ the set of all $p\times p$ PSD matrices of rank $K$. Denote by $a\vee b=\max\{a,b\}$.

\section{The Manifold of Restricted PSD Matrices}\label{sec:2}

In this section, we briefly recap the geometry for restricted PSD matrices \citep{Neuman21}. To start with, any PSD matrix $\bA\in S(p,K)$ has a unique Cholesky decomposition $\bA=\bL\bL^\top$ such that $\bL\in\RR^{p\times p}$ is a lower triangular matrix and has precisely $K$ positive diagonal elements and $p-K$ zero columns. The $j$th column of $\bL$ is zero if and only if the $j$th column of $\bA$ is linearly dependent on the previous $j-1$ columns of $\bA$. Thus, we can rewrite $\bA=\bN\bN^\top$, where $\bN\in\RR^{p\times K}$ consists of $K$ non-zero columns of $\bL$ without changing the order. Note that $\bN$ is mock lower triangular, i.e., $\bN_{ij}=0$ if $i<j$. We refer to $\bN$ as the {\it reduced Cholesky factor} of $\bA$. To further develop a geometric structure, \citet{Neuman21} consider the restricted subset $S^*(p,K)$ of $S(p,K)$ such that the first $K$ columns of $\bA\in S^*(p,K)$ are linearly independent.
The set of all reduced Cholesky factors of matrices in $S^*(p,K)$ is denoted by $\cL^*(p,K)$, which is equivalent to the set of all mock lower triangular matrices in $\RR^{p\times K}$ with positive diagonal elements.  \citet{Neuman21} impose a Riemannian structure on $S^*(p,K)$ and $\cL^*(p,K)$ such that the following mappings are isometric,
\#
&\frak{h}: S^*(p,K)\mapsto \cL^*(p,K), \, \bA\mapsto \bN,\label{equ:2.1}\\
&\frak{g}: \cL^*(p,K)\mapsto \cL(p,K), \, \bN\mapsto \bN',\label{equ:2.2}
\#
where $\bN=\frak{h}(\bA)$ is the reduced Cholesky factor of $\bA$, $\cL(p,K)=\{\bN'\in\RR^{p\times K}:\bN'_{ij}=0,i<j\}$ is endowed with a Euclidean structure, and $\bN'=\frak{g}(\bN)\in\cL(p,K)$ is defined by $\bN'_{ii}=\log(\bN_{ii}),\forall\, i$ and $\bN'_{ij}=\bN_{ij},\forall\, i>j$. We refer to $\bN'=\mathfrak{g}\circ\mathfrak{h}(\bA)$ as the reduced log-Cholesky factor of $\bA$. 
	%When $p=K$, the metric on $S^*(p,K)$ is reduced to the Log-Cholesky metric studied in \citet{Lin19}. 
	The Karcher mean $\tilde\bA$ of $M$ restricted PSD matrices $\{\bA^m\}_{m=1}^M\subset S^*(p,K)$ has a closed-form solution, which is given by
	\#\label{equ:2.3}
	\tilde\bA=\frak{h}^{-1}\circ\frak{g}^{-1}(\frac{1}{M}\sum_{m=1}^M\frak{g}\circ\frak{h}(\bA^m)).
	\#
	The algorithm computing $\tilde\bA$ is referred to as the Low Rank Cholesky (LRC) algorithm \citep{Neuman21}.

\section{Statistical Analysis of the Karcher Mean}\label{sec:3}
In this section, we provide the first statistical analysis of the Karcher mean under an intrinsic model on the restricted PSD manifold. Then we consider a general signal-plus-noise model under which a deterministic error bound of the Karcher mean is given.

\subsection{An intrinsic model}\label{sec:3.1}

Inspired by the isometry stated in equations \eqref{equ:2.1} and \eqref{equ:2.2} between the manifold $S^*(p,K)$ of restricted PSD matrices and the Euclidean space $\cL(p,K)$, we propose the following intrinsic model. Suppose $\bA\in S^*(p,K)$ is the signal matrix and denote by $\bN'=\mathfrak{g}\circ\mathfrak{h}(\bA)$ its reduced log-Cholesky factor. The observations $\{\bA^m\}_{m=1}^M$ are generated as follows:
\#
\bA^m=\mathfrak{h}^{-1}\circ\mathfrak{g}^{-1}(\bN'+\bE^m),\quad  m=1,\ldots,M,\label{equ:3.1}
\#
where $\{\bE^m\}_{m=1}^M\subset \cL(p,K)$ are independent and the lower triangular entries of $\bE^m$ are independent normal variables with mean zero and variance $\sigma^2$.  
Under this intrinsic model, the Karcher mean  of $\{\bA^m\}_{m=1}^M$ can be rewritten as
\#
\tilde \bA=\mathfrak{h}^{-1}\circ\mathfrak{g}^{-1}(\bN'+\frac{1}{M}\sum_{m=1}^M\bE^m).\label{equ:3.2}
\#
Using measure concentration, we can obtain a non-asymptotic error bound for the Karcher mean $\tilde \bA$.

\begin{theorem}[Intrinsic Model]\label{thm:3.1}
    {\it Suppose $\bA\in S^*(p,K)$ is the signal matrix and assume $\norm{\bA}_2\leq C$ for some constant $C>0$. Assume samples $\{\bA^m\}_{m=1}^M$ are generated from the intrinsic model \eqref{equ:3.1} and denote by $\tilde\bA$  the Karcher mean of $\{\bA^m\}_{m=1}^M$. Then there exist some constants $c_1,c_2>0$ such that the following inequality
    \#
    \norm{\tilde\bA - \bA}_{\rm F}\leq \sqrt{\frac{c_2pK\sigma^2}{M}}\label{equ:3.3}
    \#
    holds with probability at least $1-e^{-c_1pK}$.}
\end{theorem}

\begin{remark}
    {\it It is worth noting that \eqref{equ:3.3} achieves the optimal rate $M^{-1/2}$. In addition, it only depends on the intrinsic dimension $\cO(pK)$ of the manifold, which can be much smaller than the ambient dimension $p^2$.}
\end{remark}

\subsection{A general signal-plus-noise model}\label{sec:3.2}

The intrinsic model may be too restricted, so this subsection introduces a general signal-plus-noise model and then provides a deterministic characterization of the Karcher mean. An application of this deterministic error bound to the distributed PCA problem will be given in Section \ref{sec:4}. Similar to the intrinsic model, we denote by $\bA\in S^*(p,K)$ the signal matrix and $\bN=\mathfrak{h}(\bA)$ its reduced Cholesky factor. The observations $\{\bA^m\}_{m=1}^M\subset S^*(p,K)$ are given by
\#
\bA^m=(\bN+\bE^m)(\bN+\bE^m)^\top,\label{equ:3.4}
\#
where $\bE^m\in\mathbb{R}^{p\times K}$ represents the $m$-th noise matrix. Here $\bE^m$ is not necessarily a mock lower triangular matrix, so the model is quite general. Also, the reduced Cholesky factor of $\bA^m$ is not necessarily $\bN+\bE^m$, but rather $(\bN+\bE^m)\bQ^m$ for some orthogonal matrix $\bQ^m\in\cO_{K\times K}$. Denote by $\bN^m$ the reduced Cholesky factor of $\bA^m$.

To characterize the Karcher mean \eqref{equ:2.3} of $\{\bA^m\}_{m=1}^M$, we first establish a linear perturbation expansion of QR decomposition below.

\begin{lemma}[Linear Perturbation Expansion]\label{lma:3.2}
		{\it 
			Suppose $\bQ\in\cO_{K\times K}$ and $\bR\in\RR^{K\times K}$ is a lower triangular matrix with positive diagonal elements. Given a noise matrix $\bE\in\RR^{K\times K}$, there exist a unique orthogonal matrix $\widecheck{\bQ}\in\cO_{K\times K}$ and a lower triangular matrix $\widecheck{\bR}\in\RR^{K\times K}$ with non-negative diagonal elements such that $\widecheck{\bR}\widecheck\bQ=\bR\bQ+\bE$. When $\epsilon_0=\norm{\bE}_{\max}$ is sufficiently small, we have
		\$
		\widecheck\bQ&=\bQ+ f_{\bR}(\bE\bQ^\top)\bQ+\cO_{\max}(\epsilon_0^2),\\
		\widecheck\bR&=\bR+\bE\bQ^\top-\bR f_{\bR}(\bE\bQ^\top)+\cO_{\max}(\epsilon_0^2),
		\$
		where $f_{\bR}:\RR^{K\times K}\mapsto \RR^{K\times K}$ is given by
		\$
		&f_{\bR}(\bE)=\cU(\bR^{-1}\bE)-(\cU(\bR^{-1}\bE))^\top,\\
		&\cU(\bP)_{ij}=\bP_{ij},i<j,\quad \cU(\bP)_{ij}=0,\textnormal{otherwise}.
		\$
	}
	\end{lemma}
	
It is worth emphasizing the following properties of $f_{\bR}$. First, $f_{\bR}$ is linear in its argument, i.e., $f_{\bR}(a\bE+b\bF)=af_{\bR}(\bE)+bf_{\bR}(\bF)$ for any  $a,b\in\RR$ and $\bE,\bF\in\RR^{K\times K}$. Second, $f_{\bR}(\bE)$ is a skew-symmetric matrix, i.e., $(f_{\bR}(\bE))^\top=-f_{\bR}(\bE)$. Last, $f_{\bR}$ is bounded in the sense that $\norm{f_{\bR}(\cdot)}_{\rF}\leq\sqrt{2}\norm{\bR^{-1}}_2\norm{\cdot}_{\rF}$. 
Similar first-order perturbation theories exist in the literature for QR, Cholesky, and LU factorization \citep{Chang96,Stewart97,Stewart77,chang1997perturbation}, but none of them provides a linear perturbation expansion with a max-norm control on the remainder term, which is necessary for our development of the error bound on the Karcher mean and the subsequent applications in distributed PCA. 

Now we are ready to present a deterministic characterization of the Karcher mean. For convenience, we write $\bN=(\bR^\top\ \bB^\top)^\top$ and $\bE^{m}=(\bE^{1,m^\top}\ \bE^{2,m^\top})^\top$ such that $\bR,\bE^{1,m}\in\RR^{K\times K}$ and $\bB,\bE^{2,m}\in\RR^{(p-K)\times K}$. Note that $\bR$ is a lower triangular matrix with positive diagonal elements since $\bN\in\cL^*(p,K)$. In the following theorem, we will show that when $\epsilon_0=\max_m\norm{\bE^m}_{\max}$ is sufficiently small, the reduced Cholesky factor $\tilde\bN$ of $\tilde\bA$ differs from $\bN$ by a term linear in $\frac{1}{M}\sum_{m=1}^M\bE^m$ and an extra term of order $\cO_{\max}(\epsilon_0^2)$. Recall that $S^*(p,K)$ is the manifold of restricted PSD matrices.

\begin{theorem}[Karcher Mean on $S^*(p,K)$]\label{thm:3.3}
		\textit{When $\epsilon_0=\max_m\norm{\bE^m}_{\max}$ is sufficiently small, the reduced Cholesky factor $\tilde\bN$ of the Karcher mean $\tilde\bA$ of $\{\bA^m=(\bN+\bE^m)(\bN+\bE^m)^\top\}_{m=1}^M$ on $S^*(p,K)$ is 
		\$
		\tilde\bN&=\bN+\frac{1}{M}\sum_{m=1}^M\bE^m-\bN f_{\bR}(\frac{1}{M}\sum_{m=1}^M\bE^{1,m})\\&\quad+\cO_{\max}(\epsilon_0^2),
		\$
		where $f_{\bR}(\cdot)$ is given in Lemma \ref{lma:3.2}.}
	\end{theorem}

From Theorem \ref{thm:3.3}, one may easily derive a deterministic upper bound on $\norm{\tilde N-N}_{\rF}$ using the triangular inequality, which depends on $\norm{\frac{1}{M}\sum_{m=1}^M\bE^m}_{\rF}$ and $pK\epsilon_0^2$.

\begin{corollary}\label{corollary:3.4}
    {\it Under the same conditions of Theorem~\ref{thm:3.3}, if $\norm{\bN}_2\leq C$ and $\norm{\bR^{-1}}\leq C$ for some constant $C>0$, then we have
    \$
    \norm{\tilde\bN-\bN}_{F}\leq\cO(\norm{\frac{1}{M}\sum_{m=1}^M\bE^m}_{\rF})+\cO(pK\epsilon_0^2).
    \$}
\end{corollary}

In applications such as distributed PCA, the Frobenius norm of the average $\frac{1}{M}\sum_{m=1}^M\bE^m$ is much smaller than that of $\bE^m$. Thus, 
by Corollary~\ref{corollary:3.4}, 
the Karcher mean $\tilde\bN$ is a better approximation of $\bN$ than any $\bN^m$ (the reduced Cholesky factor of $\bA^m$).

%Using  Lemma \ref{lma:3.2}, we can easily show that when $\epsilon_0=\norm{\bE^m}_{\max}$ is sufficiently small,
%    \$
%    \bN^m=\bN+\bE^m-\bN f_{\bR}(\bE^{1,m})+\cO_{\max}(\epsilon_0^2),
%    \$
%    where $\bN^m$ is the reduced Cholesky factor of $\bA^m=(\bN+\bE^m)(\bN+\bE^m)^\top$ and both $\bN=(\bR^\top\ \bB^\top)^\top$ and $\bE^m=(\bE^{1,m^\top}\ \bE^{2,m^\top})^\top$ are $(K,p-K)$-partitions.

\section{Applications to Distributed PCA}\label{sec:4}

This section applies Theorem~\ref{thm:3.3} to show that the distributed PCA algorithm, LRC-dPCA proposed by \cite{Neuman21}, achieves the same performance as the full sample PCA when the local sample size is sufficiently large. 

\subsection{Distributed PCA and LRC-dPCA}\label{sec:4.1}
\begin{algorithm}[tb]
		\caption{LRC-dPCA}
		\label{alg:1}
		
		\begin{algorithmic}
		    \STATE {\bfseries Input:} $\{\hat\bSigma^m=\frac{1}{n}\sum_{i}\bx_i^m\bx_i^{m\top}\}_{m=1}^M$, $K$\;
			\STATE {\bfseries Output:} $\tilde\bV$\;
			\STATE Compute $\hat\bV^m$ and $\hat\bLambda^m$ of $\hat\bSigma^m$ and communicate them to a central server\;
			%\STATE Determine a suitable index set of size $K$ and view them as the first $K$ rows/columns\;
			\STATE Compute the Karcher mean $\tilde\bA$ of $\hat\bV^m(\hat\bLambda^m)^2\hat\bV^{m\top}$ on the manifold of restricted PSD matrices\;
			\STATE Compute the top $K$ eigenspace $\tilde\bV$ of $\tilde \bA$.
		\end{algorithmic}
	\end{algorithm}

We start with the distributed PCA setting as well as the LRC-dPCA algorithm. For simplicity, we consider a balanced setting, in which we have $M$ machines and the $m$-th machine has $n$ samples $\{\bx_i^m\}_{i=1}^{n}\subset\RR^p$. Denote by $N=Mn$ the total number of samples. Assume all samples are $\textnormal{i.i.d.}$ sub-Gaussian with mean $\zero$ and covariance $\bSigma$. 
\begin{definition}[sub-Gaussian]
	\textit{We say a random vector $\bx\in\RR^p$ is sub-Gaussian with mean $\zero$ and covariance $\bSigma$ if $\bz=\bSigma^{-1/2}\bx$ is sub-Gaussian with mean $\zero$  and covariance $\bI_p$, i.e., there exists a constant $\sigma>0$ such that the following inequality holds,
	\$
	\EE[e^{\lambda \langle\bu,\bz\rangle}]\leq e^{\frac{\lambda^2\sigma^2}{2}},\quad \forall \lambda\in\RR,\forall\bu\in\RR^p,\norm{\bu}_2=1.
	\$}	
	\end{definition}
\begin{remark}
	\textit{ \citet{Fan19} and \citet{Bhaskara19} use the following \emph{equivalent} definition of a sub-Gaussian vector: $\bx\in\RR^d$ is sub-Gaussian with mean $\zero$ and covariance $\bSigma$ if there exists a constant $C>0$ such that $\norm{\bu^\top\bx}_{\psi_2}\leq C\sqrt{\EE(\bu^\top\bx)^2},\forall\bu\in\RR^d$. For more information on the equivalent definitions of sub-Gaussian vectors, one may refer to \citet{Vershynin12}.}
	\end{remark}
	
	Given a positive integer $K$, the goal is to compute the top $K$ eigenspace of $\bSigma$ using all data on $M$ machines with small communication cost. We consider the LRC-dPCA algorithm proposed by \cite{Neuman21}, collected in Algorithm \ref{alg:1}. Following this algorithm, we first compute $\hat\bSigma^m=\frac{1}{n}\sum_{i=1}^n\bx_i^m\bx_i^{m\top}$ on each local machine and then compute the top $K$ eigenvectors $\hat\bV^m=(\hat\bv_1^m,\ldots,\hat\bv_K^m)\in\cO_{p\times K}$ and eigenvalues $\hat\bLambda^m=\diag(\lambda_1^m,\ldots,\lambda^m)$ of $\hat\bSigma^m$. After communicating these local estimators $\hat\bV^m,\hat\bLambda^m$ to a central server, we compute the Karcher mean $\tilde\bA$ of $\{\hat\bV^m(\bLambda^m)^2\hat\bV^{m\top}\}_{m=1}^M$ on the manifold of restricted PSD matrices\footnote{Here we choose $(\hat\bLambda^m)^2$ rather than $\hat\bLambda^m$ only for technical reasons in the theoretical proofs.}. Finally, the top $K$ eigenvectors $\tilde\bV\in\cO_{p\times K}$ of $\tilde\bA$ is returned.

\subsection{Theoretical analysis}\label{sec:4.2}

A statistical analysis of the LRC-dPCA algorithm is missing in its original paper \citep{Neuman21}. In this subsection, we will utilize our deterministic characterization of the Karcher mean on $S^*(p,K)$, i.e., Theorem~\ref{thm:3.3}, to show that given sufficiently large sub-sample size, LRC-dPCA matches the performance of the full sample PCA. 
Denote by $\bV=(\bv_1,\ldots,\bv_K)\in\cO_{p\times K}$ and $\bLambda=\diag(\lambda_1,\ldots,\lambda_K)$ the top $K$ eigenvectors and eigenvalues of $\bSigma$, respectively. To ensure the uniqueness of $\textnormal{span}(\bV)$, we assume $\Delta_K=\lambda_K(\bSigma)-\lambda_{K+1}(\bSigma)>0$. Write $\bA=\bV\bLambda^2\bA^\top$ and assume the first $K$ columns of $\bA$ are linearly independent, i.e., $\bA\in S^*(p,K)$. When the subsample size $n$ is sufficiently large, we will show that $\hat\bA^m=\hat\bV^m(\hat\bLambda^m)^2\hat\bV^{m\top}$ also belongs to $S^*(p,K)$ with high probability. Here $\hat\bV^m$ and $\hat\bLambda^m$ denote the top $K$ eigenvectors and eigenvalues of $\hat\bSigma^m$ respectively. Denote by $\tilde\bA$ the Karcher mean of $\{\hat\bA^m\}_{m=1}^M$ on $S^*(p,K)$.  We further  denote by $\bN,\hat\bN^m,\tilde\bN$ the reduced Cholesky factors of $\bA,\hat\bA^m,\tilde\bA$. In addition, we define $\bQ^*\in\cO_{K\times K}$ by the equality $\bN=\bV\bLambda\bQ^*$.

In the rest of this subsection, we will apply Theorem~\ref{thm:3.3} to study the properties of $\tilde \bA$. First,  we show that $\{\hat\bA^m\}_{m=1}^M$ follow the general signal-plus-noise model \eqref{equ:3.4}.

\begin{lemma}\label{lma:4.3}
	     Let $\hat\bE^m=\hat\bSigma^m\hat\bV^m\hat\bH^m\bQ^*-\bSigma\bV\bQ^*$, where $\bH^m=\hat\bV^{m\top}\bV$ and $\hat\bH^m=\textnormal{sgn}(\bH^m)\overset{\rm def}{=}\bU_1\bU_2^\top$ with $\bU_1,\bU_2$ given by the singular value decomposition  $\bH^m=\bU_1\bGamma\bU_2^\top$  of $\bH^m$. Then $\hat\bA^m=(\bN+\hat\bE^m)(\bN+\hat\bE^m)^\top$.
	\end{lemma}

Let us make several remarks on $\hat\bE^m$.
It is well-known that 
\$
\hat\bH^m=\argmin_{\bO\in\cO_{K\times K}}\norm{\hat\bV^m\bO-\bV}_{\rF}
\$ 
and thus $\hat\bV^m\hat\bH^m$ is a good estimator of $\bV$~\citep{Chen20}. 
%To apply Theorem \ref{thm:3.3}, we will upper bound the max norm $\epsilon_0=\max_m\norm{\hat\bE^m}_{\max}$ in the next step. 
%This is based on the first-order expansion of $\bE^m$ in terms of $\cE^m=\hat\bSigma^m-\bSigma$. 
Furthermore, by Lemma \ref{lma:a2}, when $\epsilon=\max_m\norm{\cE^m}_2/\Delta_K\leq1/10$ with $\cE^m=\hat\bSigma^m-\bSigma$, $\hat\bV^m\hat\bH^m$ has the following first-order expansion around $\bV$,
	\#\label{equ:4.1}
	\hat\bV^m\hat\bH^m=\bV+g(\cE^m\bV)+\cO_{\rF}(\epsilon^2),
	\#
	where $g$ is a \emph{linear} function defined in Lemma \ref{lma:a2}. Substituting \eqref{equ:4.1} into the definition of $\hat\bE^m$, we obtain the following linear expansion of $\hat\bE^m$ in terms of $\cE^m$,
	\#\label{equ:4.2}
	\hat\bE^m=\cE^m\bV\bQ^*+\bSigma g(\cE^m\bV)\bQ^*+\cO_{F}(\epsilon^2).
	\#
%	This allows us to show that $\max_{m}\norm{\hat\bE^m}_{\max}$ is sufficiently small with high probability given sufficiently large subsample size $n$.
Since $g$ is {linear} in its argument, the leading term of $\frac{1}{M}\sum_{m=1}^M\hat\bE^m$ is linear in $\frac{1}{M}\sum_{m=1}^M\cE^m$. This enables an upper bound for $\norm{\frac{1}{M}\sum_{m=1}^M\hat\bE^m}_{\rF}$, provided by the following lemma.
	\begin{lemma}[Bounding $\norm{M^{-1}\sum_{m=1}^M\hat\bE^m}_{\rF}$]\label{lma:4.4}
	\textit{Suppose $\Delta_K>0$ and $\norm{\bSigma}_2$ is bounded. Let $\cE^m=\hat\bSigma^m-\bSigma$ and $\epsilon=\max_m\norm{\cE^m}_2/\Delta_K$. When $\epsilon\leq 1/10$, the following bound
	\$
	\norm{\frac{1}{M}\sum_{m=1}^M\hat\bE^m}_{\rF}\leq C\norm{\frac{1}{M}\sum_{m=1}^M\cE^m}_2+\cO(\epsilon^2)
	\$
	holds for some constant $C>0$.}
	\end{lemma}
	
	%Thus, by Theorem \ref{thm:3.3}, $\tilde\bN$ is expected to be a better estimator of $\bN$ than any $\hat\bN^m$. 
	To apply Theorem \ref{thm:3.3}, we also need to upper bound the max norm $\epsilon_0=\max_m\norm{\hat\bE^m}_{\max}$. Again, this is based on the first-order expansion \eqref{equ:4.2} of $\bE^m$.
	% in terms of $\cE^m$. 
	
	\begin{lemma}[Bounding $\max_m\norm{\hat\bE^m}_{\max}$]\label{lma:4.5}
	\textit{Assume $\Delta_K>0$ and $\norm{\bSigma}_2$ is bounded. When $\epsilon=\max_m\norm{\cE^m}_2/\Delta_K\leq 1/10$, we have with probability at least $1-2Me^{-C_1n\delta_1^2}-Me^{-C_2\sqrt{\delta_2n/r}}$ that
	\$
	\max_m\norm{\hat\bE^m}_{\max}\leq C_3\sqrt{\frac{\log(p)}{n}}+\delta_1+\delta_2,
	\$
	for some constants $C_1,C_2,C_3>0$ and $r=\textnormal{Tr}(\bSigma)/\lambda_1(\bSigma)$. In addition, when $n\gtrsim \log^3(pM)r^2$, we have with probability at least $1-2p^{-1}$ that
	\$
	\max_m\norm{\hat\bE^m}_{\max}\leq C\sqrt{\frac{\log(pM)}{n}},
	\$
	for some constant $C>0$.}
	\end{lemma}
	
		In Lemma \ref{lma:4.5}, we show that $\norm{\hat\bE^m}_{\max}\lesssim\sqrt{\log(p)/n}$ with high probability when $n\gtrsim\log^3(p)r^2$. By \eqref{equ:4.2} and Lemma~\ref{lma:a1}, we can show that $\norm{\hat\bE^m}_{\rF}\lesssim\sqrt{p/n}$ with high probability. The upper bound on $\norm{\hat\bE^m}_{\max}$ is thus smaller by a factor of $\sqrt{p/\log(p)}$ than the upper bound on $\norm{\hat\bE^m}_{F}$. This implies that $\hat\bE^m$ is delocalized across the entries. Moreover, Lemma \ref{lma:4.5} implies that when we apply Theorem \ref{thm:3.3} to the LRC-dPCA algorithm, the remainder term $\cO_{\max}(\epsilon_0^2)$ is negligible compared to the leading term $\frac{1}{M}\sum_{m=1}^M\hat\bE^m-\bN f_{\bR}(\frac{1}{M}\sum_{m=1}^M\hat\bE^{1,m})$. This provides the last key ingredient to the following theorem, which gives an upper bound for $\norm{\tilde\bN-\bN}_{F}$. Here $\tilde\bN$ is the reduced Cholesky factor of the Karcher mean $\tilde\bA$.
		
	\begin{theorem}[Bounding $\norm{\tilde\bN-\bN}_{\rF}$]\label{thm:4.6}
{\it	Assume $\Delta_K>0$ and $\norm{\bSigma}_2$ is bounded. 
Partition $\bN=(\bR^\top\ \bB^\top)^\top$ such that $\bR\in\RR^{K\times K}$ and $\bB\in\RR^{(p-K)\times K}$ and assume $\norm{\bR^{-1}}_2\leq C$ for some constant $C>0$.
When $\epsilon=\max_m\norm{\cE^m}_2/\Delta_K\leq 1/10$ and $\epsilon_0=\max_{m}\norm{\hat\bE^m}_{\max}$ is sufficiently small, the following bound
	\$
	\norm{\tilde\bN-\bN}_{\rF}\leq \cO\left(\norm{\frac{1}{M}\sum_{m=1}^M\cE^m}_2\right)+\cO(\epsilon^2)+\cO(\sqrt{p}\epsilon_0^2)
	\$
	holds. Define $r=\textnormal{Tr}(\bSigma)/\lambda_1(\bSigma)$, $\tilde r_1=(\log^2(pM)r)\vee(\log(pM)\sqrt{p})$ and $\tilde r_2=\sqrt{p}\log^4(pM)r^2$. Then we have with probability at least $1 - 4p^{-1}$ that
	\$
    \norm{\tilde\bN-\bN}_{\rF}\leq \cO\left(\frac{\log(p)\sqrt{r}}{\sqrt{Mn}}\right)+\cO\left(\frac{\tilde r_1}{n}\right)+\cO\left(\frac{\tilde r_2}{n^2}\right).
    \$  
    When $n\gtrsim \tilde r_2/\tilde r_1$, the third term is negligible. When we further assume $n\gtrsim M\tilde r_1^2/(\log^2(p)r)$, the upper bound reduces to
    \$
    \norm{\tilde\bN-\bN}_{\rF}\leq \cO\left(\frac{\log(p)\sqrt{r}}{\sqrt{Mn}}\right).
    \$}
	\end{theorem}
	
	Theorem \ref{thm:4.6} shows that given sufficiently large local sample size, i.e., $n\gtrsim M\tilde r_1^2/(\log^2(p)r)$, $\tilde\bN$ is as good as the full sample estimator of $\bN$ in terms of the Frobenius norm. Moreover, $\norm{\tilde\bN-\bN}_{\rF}$ is of the same order as $\norm{M^{-1}\sum_{m=1}^M\cE^m}_2$ (see Lemma \ref{lma:a1}). Note that the singular vectors of $\bN$ are equal to $\bV$, the singular values of $\bN$ are equal to $\bLambda$, and LRC-dPCA uses the singular vectors of $\tilde\bN$ as an estimator of $\bV$. Then it follows from  Wedin's sin($\bTheta$) theorem 
	%and Davis-Kahan sin($\bTheta$) theorem 
	\citep{Chen20} that LRC-dPCA and full sample PCA share the same performance in eigenvector estimation.

\begin{remark}\label{rmk:5}
	\textit{Similar to Lemma \ref{lma:4.5}, we can show that $\norm{\hat\bV^m\hat\bH^m-\bV}_{\max}\lesssim\sqrt{{\log(p)}/{n}}$ with high probability when $n\gtrsim \log^3(p)r^2$. Compared to the upper bound $\norm{\hat\bV^m\hat\bH^m-\bV}_{F}\lesssim\sqrt{{p}/{n}}$, the max norm bound again implies that the residual matrix $\hat\bV^m\hat\bH^m-\bV$ does not concentrate on a few coordinates. This has connections to the infinity norm eigenvector perturbation theory \citep{Fan18,Chen20,Abbe20,damle2020uniform,cape19b}. However, most applications in their works require incoherence conditions on the eigenvectors. In contrast, we do not require such conditions. }
\end{remark}

\subsection{Manifold selection}

As one may notice, $\bA=\bV\bLambda^2\bV^\top$ may not belong to $S^*(p,K)$, i.e., the first $K$ columns of $\bA$ may be linearly dependent. If we decompose $\bA=\bF\bF^\top$ for some $\bF\in\RR^{p\times K}$ and write $\bF=(\bF_1^\top\ \bF_2^\top)^\top$ with $\bF_1\in\RR^{K\times K}$ and $\bF_2\in\RR^{(p-K)\times K}$. Then the smallest singular value  $\sigma_{\min}(\bF_1)$ of $\bF_1$ may be zero or very small depending on $p$. In these cases, the condition $\norm{\bR^{-1}}_2\leq C$ for some constant $C>0$ in Theorem~\ref{thm:4.6} may not hold, and it is not suitable to directly use the manifold $S^*(p,K)$ in the LRC-dPCA algorithm. 

To fix this issue, we will utilize  $\frac{p!}{(p-K)!}$ cousins of the manifold $S^*(p,K)$, or equivalently $\cL^*(p,K)$. Let us introduce these cousin manifolds first. Recall that $\cL^*(p,K)$ consists of $\bN\in\RR^{p\times K}$ such that $\bN_{1:K,1:K}$ is a lower triangular matrix with positive diagonal elements. Here $\bN_{1:K,1:K}\in\RR^{K\times K}$  represents the sub-matrix of $\bN$ with row index $[1,\ldots,K]$ and column index $[1,\ldots,K]$. Let $\cI=[i_1,\ldots,i_K]$ be an ordered index set of size $K$. A cousin $\cL^*_{\cI}(p,K)$ of $\cL^*(p,K)$ consists of $\bN\in\RR^{p\times K}$ such that $\bN_{\cI,1:K}$ is a lower triangular matrix with positive diagonal elements. Similarly, we define $S^*_{\cI}(p,K)$ as the set of all matrices in $S(p,K)$ with the $\cI$-th rows linearly independent. Similar to the relationship between $S^*(p,K)$ and $\cL^*(p,K)$, for any $\bA\in S^*_{\cI}(p,K)$, there exists a unique element $\bN\in\cL^*_{\cI}(p,K)$ such that $\bA=\bN\bN^\top$. Also, we define the Riemannian structure on $S^*_{\cI}(p,K)$ and $\cL^*_{\cI}(p,K)$ in a way similar to \eqref{equ:2.1} and \eqref{equ:2.2}. In addition, all theory established in Section~\ref{sec:3} and \ref{sec:4} can be rephrased in the language of $S^*_{\cI}(p,K)$. The only difference is that the row index set $[1,\dots,K]$ is replaced by $\cI$.

Now we are in a position to solve the challenge raised at the beginning of  this subsection. If $\bA=\bV\bLambda^2\bV^\top$ does not belong to $S^*(p,K)$, then we should choose a suitable ordered index set $\cI$ rather than $[1,\ldots,K]$, and then apply the LRC-dPCA algorithm on the manifold $S^*_{\cI}(p,K)$. Motivated by the condition $\norm{\bR^{-1}}\leq C$ in Theorem~\ref{thm:4.6}, we propose the $\texttt{find\_index}$ method in Algorithm \ref{alg:2}.
Given $\bV$, $\bLambda$, and $K$, the algorithm outputs an ordered index set $\cI$ of size $K$.
To avoid exhaustive search, the algorithm determines $\cI$ in a sequential manner. In the $k$th step, we choose an index $i\in[p]$ such that the $k$-by-$k$ matrix $\bT_k=\bT[c(\cI[1:(k-1)],i),c(1:k)]$ has the largest $\sigma_k(\bT_k)$ among all $p$ candidates, where $c(\cdot)$ indicates the index set. In practice when $\bV$ and $\bLambda$ is unknown, we can use $\hat\bV^1$ and $\hat\bLambda^1$ to find a suitable index set and this index set is then shared by all machines.

\begin{algorithm}[tb]
		\caption{\texttt{find\_index} in LRC-dPCA}
		\label{alg:2}
		\begin{algorithmic}
			\STATE {\bfseries Input:} $\bV,\bLambda$, $K$
			\STATE {\bfseries Output:} $\cI\subset[p]$
			\STATE Compute $\bT=\bV\bLambda$ and initialize $\cI=[0,\ldots,0]\in\ZZ^{K}$.
			\FOR{$k=1$ {\bfseries to} $K$}
			    \FOR{$i=1$ {\bfseries to} $p$}
			        \STATE Set $\bT_k = \bT[c(\cI[1:(k-1)],i),c(1:k)]$.
			        \STATE Compute score[$i$] = $\sigma_{k}(\bT_k)$.
			    \ENDFOR
			    \STATE Set $\cI[k]=\argmax_i\textnormal{score}[i]$.
			\ENDFOR
		\end{algorithmic}
	\end{algorithm}

\section{Numerical Experiments}\label{sec:5}

In this section, we present numerical experiments on three synthetic examples: averaging PSD matrices under the intrinsic model, the distributed PCA problems, and averaging PSD matrices under an extrinsic model.

\subsection{Averaging PSD matrices}

\begin{figure}[t]
    \centering
    \includegraphics[width=0.5\textwidth]{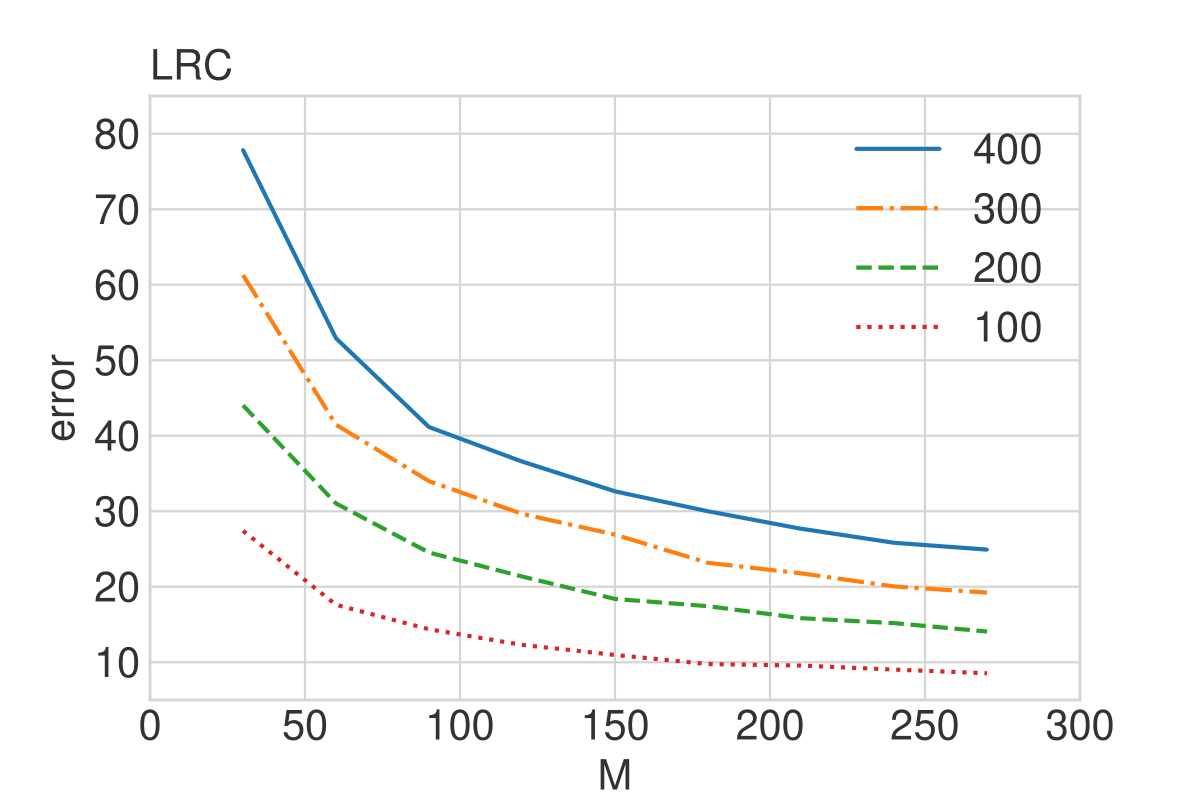}
    \includegraphics[width=0.5\textwidth]{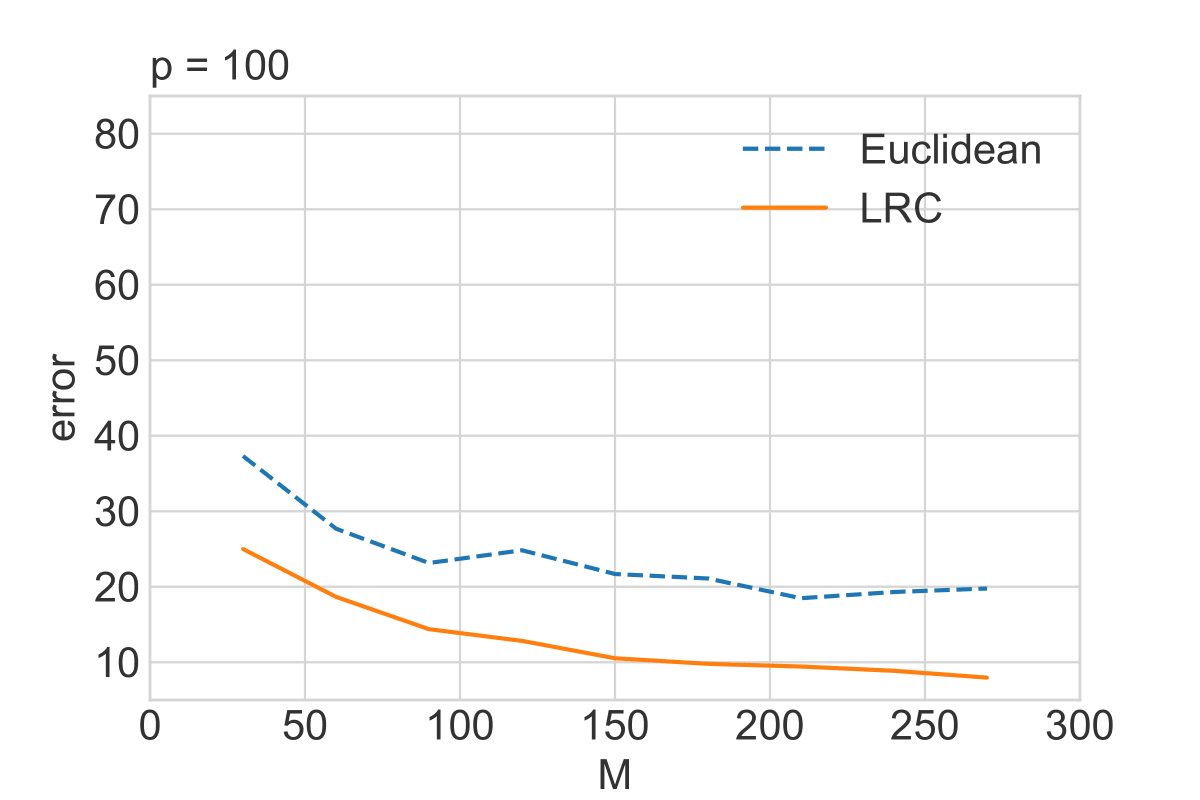}
    \caption{Averaging PSD matrices under the intrinsic model. Top figure: errors $\norm{\tilde\bA-\bA}_{\rF}$ against different $M$ and $p$ with four colored lines labeled by $p$. Bottom figure: comparisons between LRC and the Euclidean method in terms of $\norm{\tilde \bA-\bA}_{\rF}$ or $\norm{\tilde \bA^{\rm eu}-\bA}_{\rF}$ against different $M$.}
    \label{fig:1}
\end{figure}

Our first experiment is to illustrate the concentration of the Karcher mean \eqref{equ:2.3} under the intrinsic model \eqref{equ:3.1}, i.e., Theorem \ref{thm:3.1}. We set $K=5,\sigma^2=1$ and let $p$ vary across $[100,200,300,400]$ and let $M$ range from 30 to 270 with an increment of 30. For each $p$, we generate a $p\times p$ matrix $\bSigma$ with elements $\textnormal{i.i.d.}$ $\cN(0,1)$, and then take $\bA=\bV\bLambda\bV^\top$, where $\bV=(\bv_1,\ldots,\bv_K)$ and $\bLambda=(\lambda_1,\ldots,\lambda_K)$ are the top $K$ left singular vectors and singular values of $\bSigma$, respectively. Given $M$, we generate $\{\bA^m\}_{m=1}^M$ from the intrinsic model \eqref{equ:3.1}. Then the Karcher mean $\tilde\bA$ of $\{\bA^m\}_{m=1}^M$ is computed and the error $\norm{\tilde\bA-\bA}_{F}$ is reported in the top figure in Figure~\ref{fig:1}. As our theory shows, the estimation error turns smaller as $M$ increases or $p$ decreases.

In addition, we compare the Karcher mean $\tilde \bA$, referred to as LRC, with the usual Euclidean method $\tilde A^{\rm eu}$, which is defined as the best rank-$K$ approximation of $M^{-1}\sum_{m=1}^M\bA^m$. We take $p=100$ and repeat the above data generation processing. The errors $\norm{\tilde\bA-\bA}_{\rF}$ and $\norm{\tilde\bA^{\rm eu}-\bA}_{\rF}$ are reported in the bottom figure in Figure~\ref{fig:1}. As displayed in the figure, under the intrinsic model, the geometry-aware method, LRC, outperforms the Euclidean method. This justifies the intuition that for models with specific geometric structures, it is better to take that geometric information into account.

\subsection{Distributed PCA}\label{sec:5.2}

\begin{figure}[t]
    \centering
    \includegraphics[width = 0.5\textwidth]{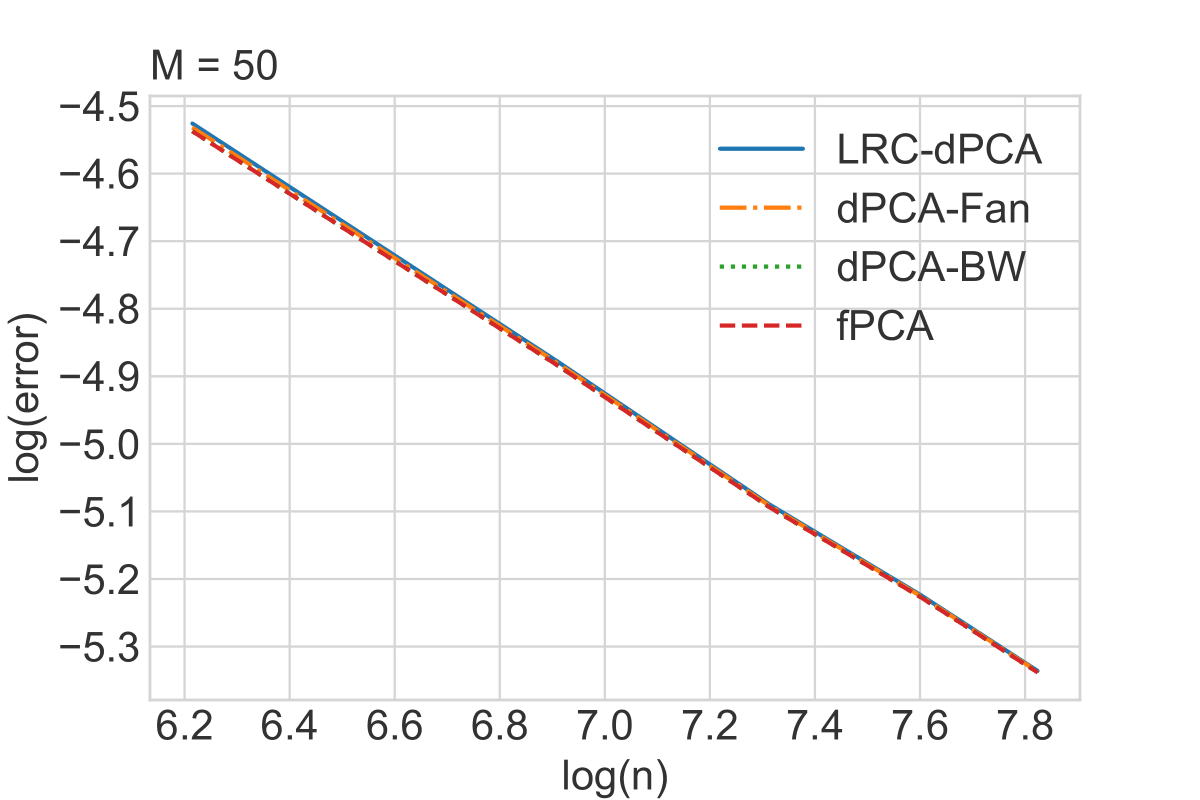}
    \includegraphics[width = 0.5\textwidth]{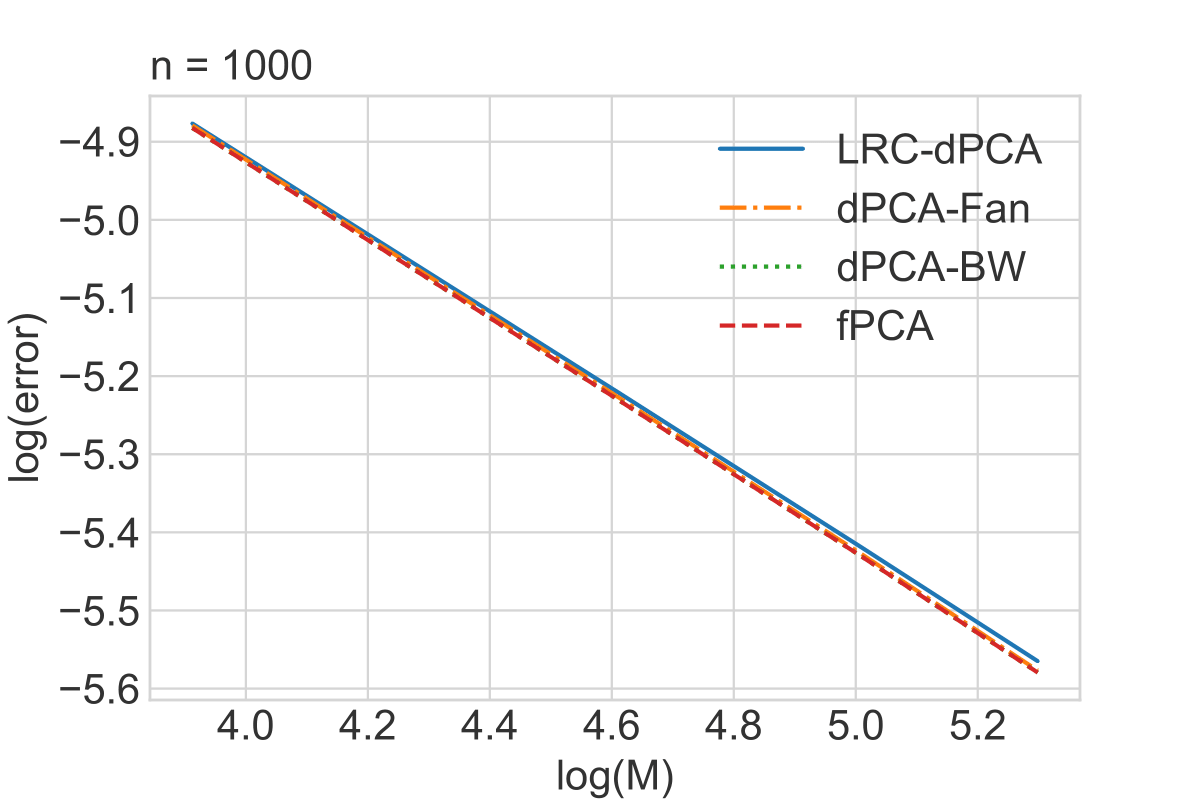}
    \caption{Comparisons of four distributed PCA algorithms, LRC-dPCA, dPCA-Fan, dPCA-BW, fPCA. Top figure: $M=50$ and $\log(\rm{error})\sim\log(n)$ is reported. Bottom figure: $n=1000$ and $\log(\rm{error})\sim\log(M)$ is displayed.}
    \label{fig:2}
\end{figure}

Our second experiment studies the Karcher mean under a general signal-plus-noise model. Specifically, we consider the distributed PCA problems and numerically verify Theorem~\ref{thm:4.6}, which shows that LRC-dPCA achieves the same performance as full sample PCA (fPCA). In our setting, $p=100$, $K=5$, the population covariance $\bSigma$ is generated by $\bSigma = \bV\bV^\top+0.3\bI_p$, where $\bV\in\RR^{p\times K}$ with elements $\rm{i.i.d.}$ $\cN(0,1)$. We first fix the number of machines $M=50$ and let the sub-sample size $n$ vary across $[500,1000,\ldots,2500]$. 
%Then we fix $n=1000$ and let $M$ vary across $[50,100,\ldots,200]$. 
On the $m$-th machine, we generate $n$ $\rm{i.i.d.}$ samples $\{\bx_{i}^m\}_{i=1}^n$ from $\cN(\zero,\bSigma)$ and compute the local sample covariance matrix $\hat\bSigma^m=\sum_{i=1}^n\bx_i^m\bx_i^{m\top}/n$. Then we apply four methods , namely fPCA, LRC-dPCA, dPCA-Fan \citep{Fan19}, and dPCA-BW \citep{Bhaskara19}, to compute the top $K$ eigenvectors of $\bSigma$. Let $\hat\bV\in\cO_{p\times K}$ be the estimated top $K$ eigenvectors. The error is defined as $\norm{\hat\bV\hat\bV^\top-\bV(\bV^\top\bV)^{-1}\bV^\top}_{F}$, which is the distance between the population projection matrix $\bV(\bV^\top\bV)^{-1}\bV^\top$
and the estimated projection matrix $\hat\bV\hat\bV^\top$. For each $n$ and each method, the experiment is repeated 100 times and the average of error is recorded. The top figure in Figure~\ref{fig:2} displays the relationship between $\log(\rm{error})$ and $\log(n)$ for all methods. It turns out that all methods share similar performance and there is a linear relationship between $\log(\rm{error})$ and $\log(n)$ with slope $-1/2$, which verifies the relationship $\rm{error}\sim n^{-1/2}$. Next, we fix the sub-sample size $n=1000$ and let $M$ vary across $[50,100,\ldots,200]$ and repeat the above procedures. The relationship between $\log(\rm{error})$ and $\log(M)$ is reported in the bottom figure of Figure~\ref{fig:2}. As it displayed, all four methods are almost the same and there is also a linear relationship between $\log(\rm{error})$ and $\log(M)$ with slope $-1/2$, which indicates $\rm{error}\sim M^{-1/2}$. Since there is no specific geometric information in the setting, it is expected that LRC-dPCA only matches (rather than surpasses) the performance of the state-of-the-art methods, dPCA-Fan, dPCA-BW, and the optimal method, fPCA.

\subsection{Averaging PSD matrices (extrinsic)}

\begin{figure}[t]
    \centering
    \includegraphics[width=0.5\textwidth]{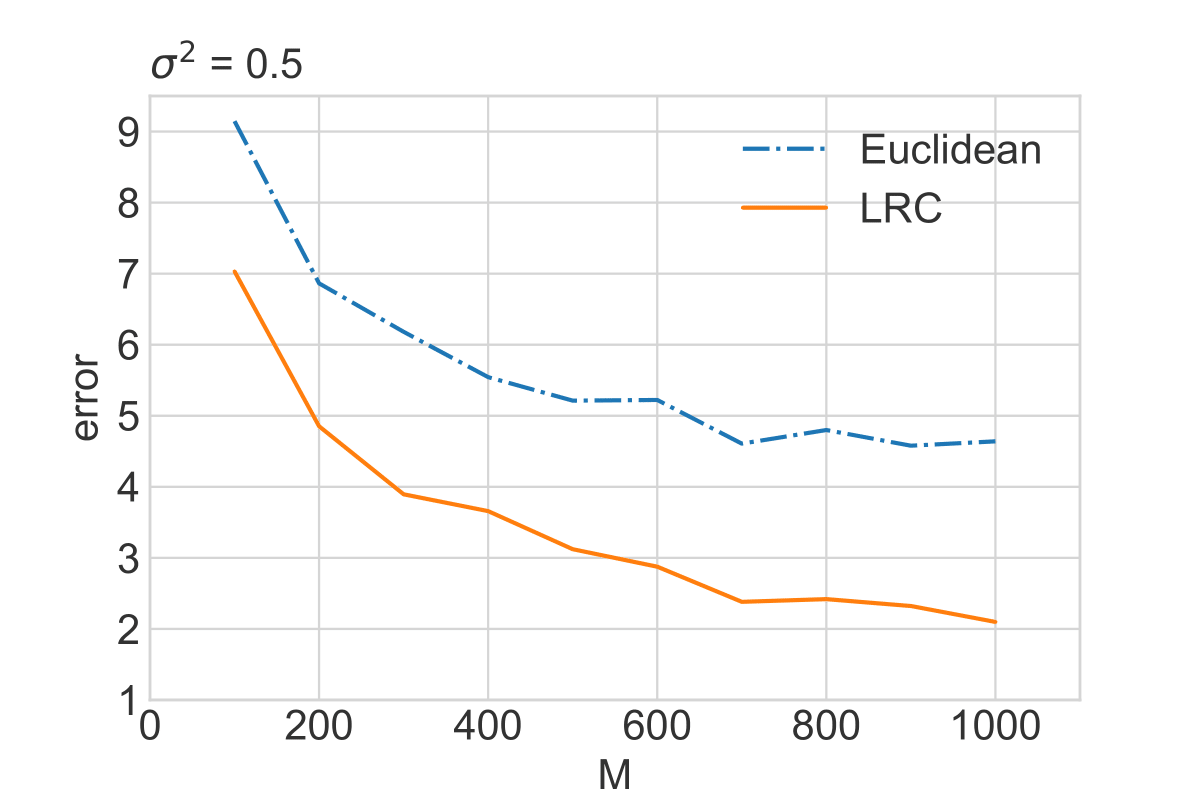}
    \includegraphics[width=0.5\textwidth]{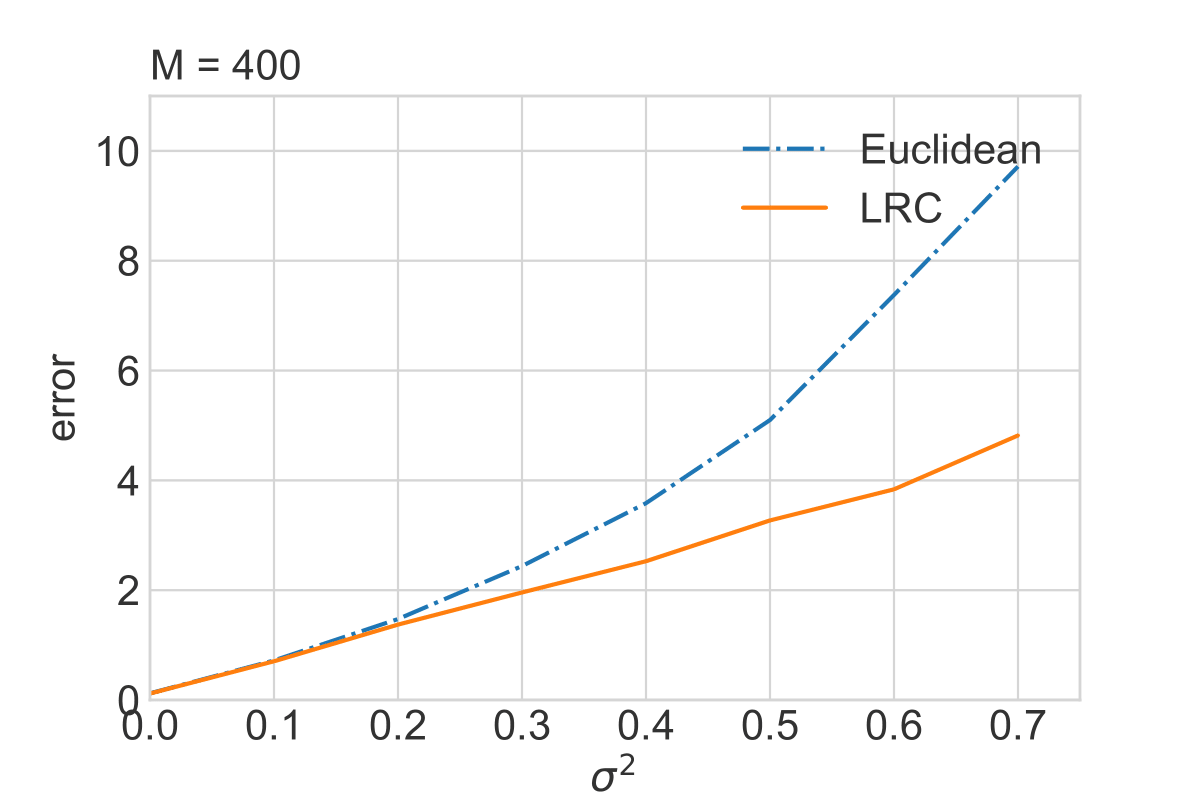}
    \caption{Comparisons of LRC and the Euclidean method in averaging PSD matrices under an extrinsic model. Top figure: $\sigma^2=0.5$ and the error against $M$ is reported. Bottom figure: $M=400$ and the error versus $\sigma^2$ is displayed.}
    \label{fig:3}
\end{figure}

Our third experiment considers another signal-plus-noise model, which adds extrinsic noises to the intrinsic model. Specifically, we set $p=100$, $K=5$, and we generate a $p\times p$ matrix $\bSigma$ with elements $\textnormal{i.i.d.}$ $\cN(0,1)$, and then take $\bA=\bV\bLambda\bV^\top$, where $\bV=(\bv_1,\ldots,\bv_K)$ and $\bLambda=(\lambda_1,\ldots,\lambda_K)$ are the top $K$ left singular vectors and singular values of $\bSigma$, respectively. Given $M$ and $\sigma^2$, we generate $\{\bA^m\}_{m=1}^M$ from the intrinsic model \eqref{equ:3.1}. Then we add extrinsic noises to $\bA^m$ as follows. For each $m$, we generate $\{\bx_i^m\}_{i=1}^{2000}$ $\rm{i.i.d.}$ from $\cN(\zero,\bA^m+0.01\bI_p)$, compute $\hat\bSigma^m=\sum_{i=1}^{2000}\bx_i^m\bx_i^{m\top}/2000$, and set $\bA'^m$ as the best rank-$K$ approximation of $\hat\bSigma^m$. We compute the Karcher mean $\tilde\bA$ of $\{\bA'^m\}_{m=1}^M$, which is referred to as LRC, and report the error $\norm{\tilde\bA-\bA}_{\rF}$. In contrast, we also apply the Euclidean method, which computes the best rank-$K$ approximation $\tilde\bA^{\rm eu}$ of $\sum_{m=1}^M\bA'^m/M$, and report the error $\norm{\tilde\bA^{\rm eu}-\bA}_{\rF}$. First, we set $\sigma^2=0.5$ and let $M$ vary across $[100,200,\ldots,1000]$. The errors of both methods are displayed in the top figure of Figure~\ref{fig:3}. As shown in the figure, the geometry-aware method, LRC, still outperforms the Euclidean method even if extrinsic noises are added to the intrinsic model. Next, we fix $M=400$ and let $\sigma^2$ range from $[0,0.1\ldots,0.7]$. The errors of both methods are shown in the bottom figure of Figure~\ref{fig:3}. Recall that $\sigma^2$ denotes the strength of intrinsic noises. The bottom figure indicates that when the intrinsic noises are small, then the geometry-aware method and the Euclidean method are comparable, but when the intrinsic noises becomes large, the geometry-aware method tends to outperform the Euclidean method. Overall, this experiment shows that, in a general signal-plus-noise model, if there exist large intrinsic noises, then it is better to utilize the geometry-aware method.

\section{Concluding Remarks}\label{sec:6}

This paper considers the geometry of restricted PSD matrices proposed by \citet{Neuman21}. In particular, we provide a non-asymptotic statistical analysis of the Karcher mean of restricted PSD matrices under an intrinsic model. Moreover, for general signal-plus-noise models, we establish a deterministic error bound concerning the Karcher mean. This is based on a linear perturbation expansion of the QR decomposition, which may be of independent interest. As an application, we use the deterministic error analysis of the Karcher mean to prove that the distributed PCA algorithm, LRC-dPCA, achieves the same performance as the full sample PCA. Motivated by the established theory, we propose a manifold selection procedure for the LRC-dPCA algorithm. Finally, we carry out three synthetic numerical experiments to verify our theories. One observation in the experiment is that if data model has certain geometric structure, then it is better to utilize the geometry-aware method.

Several interesting topics are worth of future studies. 
In manifold-valued data analysis \citep{patrangenaru2016nonparametric}, it remains to determine which statistical model is more suitable for the given data. 
For example, the highly anisotropic diffusion tensor images are modelled as PSD matrices \citep{Bonnabel13}, so it is interesting to investigate the performances of the proposed intrinsic model  for such data. 
Second, it is interesting  to extend our study to regression, classification, and clustering problems.

\bibliographystyle{apalike}
\bibliography{references}

\newpage
\onecolumn
%\appendix
\section*{APPENDIX}
\appendix
%not sure if I should put it here.
%\section{Additional Experiments}

%In this section, we numerically demonstrate the effectiveness of the \texttt{find\_index} algorithm. We set $p=100$, $K=5$, $n=1000$, and $\bSigma=\bV\bV^\top+\bI_p$, where $\bV\in\RR^{p\times K}$ with entries $\rm i.i.d.$ $\cN(0,1)$. Let $M$ vary across $[50,100,\ldots,200]$. We repeat the experiments in Section \ref{sec:5.2} except that we choose a different $\bSigma$ and additionally compare the LRC-dPCA-index algorithm, which applies the \texttt{find\_index} method in LRC-dPCA. The results are displayed in Figure \ref{fig:a1}.

%\begin{figure}[h]
%    \centering
%    \includegraphics[width = 0.7\textwidth]{figgg.pdf}
%    \caption{Illustration of the \texttt{find\_index} method.}
%    \label{fig:a1}
%\end{figure}

%As shown in Figure \ref{fig:a1}, LRC-dPCA-index does improve the LRC-dPCA algorithm by choosing a better index set.

\section{Proof of Theorem \ref{thm:3.1}}
\begin{proof}[Proof of Theorem \ref{thm:3.1}]
	Recall that the Karcher mean $\tilde\bA$ of $\{\bA^m\}_{m=1}^M$ under the intrinsic model is given by \eqref{equ:3.1}. First, we give an upper bound on the Frobenius norm of $\frac{1}{M}\sum_{m=1}^M\bE^m$. By the intrinsic model, we know $\frac{1}{M}\sum_{m=1}^M\bE^m$ is a mock lower triangular matrix with lower triangular elements $\textnormal{i.i.d.}$ $\cN(0,\sigma^2/M)$. Therefore, by the concentration of $\chi^2$ ((2.19) in \cite{Wainwright19}), for all $t\in(0,1)$, we have
	\#
	\norm{\frac{1}{M}\sum_{m=1}^M\bE^m}_{\rF}^2\leq\frac{pK\sigma^2}{M}(1+t),
	\#
	with probability at least $1-e^{-pKt^2/8}$. In a similar spirit, using union bound, we have for $t\in(0,1)$,
	\#
	\max_{i=1,\ldots,K}|\frac{1}{M}\sum_{m=1}^M\bE^m_{ii}|^2\leq\frac{\sigma^2}{M}(1+t)
	\#
	with probability at least $1-Ke^{-t^2/8}$. Here $\bE_{ii}^m$ is the $(i,i)$-th element of $\bE^m$. Thus with high probability, $\max_{i=1,\ldots,K}|\frac{1}{M}\sum_{m=1}^M\bE^m_{ii}|^2\leq1/2$ and 
	\#
	|\exp(\frac{1}{M}\sum_{m=1}^M\bE^m_{ii})-1|\leq2|\frac{1}{M}\sum_{m=1}^M\bE^m_{ii}|,
	\#
	where we use the inequality $|\exp(x)-1|\leq2|x|$ for $x\leq1/2$.
	Denote by $\bN$ the reduced Cholesky factor of $\bA$. Then it holds that $\norm{\bN}_2=\norm{\bA}_2^{1/2}\leq C^{1/2}$ for some constant $C>0$. Furthermore, by \eqref{equ:3.1}, we have for some constants $c_1,c_2>0$ that
	\#
	\norm{\tilde\bA-\bA}_{F}\leq \sqrt{\frac{c_2pK\sigma^2}{M}},
	\#
	with probability at least $1-e^{-c_1pK}$.
\end{proof}

\section{Proof of Lemma \ref{lma:3.2}}\label{apx:2}
\begin{proof}[Proof of Lemma \ref{lma:3.2}]
	The proof of this lemma is split up into three steps. First, we assume $\bQ=\bI_K$ and show that $\widecheck\bQ$ has the form of $\bI_K+\widecheck\bP+\cO_{\max}(\epsilon_0^2)$, where $\widecheck\bP\in\RR^{K\times K}$ is a skew-symmetric matrix of order $\cO_{\max}(\epsilon_0)$. Second, by taking upper triangular off-diagonal elements of $(\bR+\bE)\widecheck\bQ^\top$ as zero, we derive a closed-form expression of $\widecheck\bP$ (up to a higher-order term). Motivated by this closed-form expression, we define a function $f_{\bR}:\RR^{K\times K}\mapsto\RR^{K\times K}$ satisfying several desired conditions. For example, we have $\widecheck\bQ=\bI_K+f_{\bR}(\bE)+\cO_{\max}(\epsilon_0^2)$ and $f_{\bR}$ is linear in its argument. Third, we extend the results to the general case when $\bQ\in\cO_{K\times K}$ may differ from $\bI_K$.
	
	\textit{Step 1.} When $\epsilon_0=\norm{\bE}_{\max}$ is sufficiently small, the matrix $\bR+\bE$ is still non-singular and by QR decomposition there exists a unique orthogonal matrix $\widecheck\bQ\in\cO_{K\times K}$ such that $\widecheck\bR=(\bR+\bE)\widecheck\bQ^\top$ is a lower triangular matrix with positive diagonal elements. In this step, we will show that $\widecheck\bQ$ has a form of $\bI_K+\widecheck\bP+\cO_{\max}(\epsilon_0^2)$, where $\widecheck\bP\in\RR^{K\times K}$ is a skew-symmetric matrix of order $\cO_{\max}(\epsilon_0)$. To that end, we construct $\widecheck\bQ$ as a product of $K(K-1)/2$ rotation matrices $\{\widecheck\bQ^{ij},1\leq i<j\leq K\}$, which set the upper triangular off-diagonal elements as zero in a sequential fashion. In specific, we arrange these $K(K-1)/2$ rotation matrices in a prescribed order, i.e., $\{(1,2),\ldots,(1,K),(2,3),\ldots,(K-1,K)\}$. In this way, we may relabel $\{\widecheck\bQ^{ij},1\leq i<j\leq K\}$ as $\{\widecheck\bQ^{(s)},1\leq s\leq K(K-1)/2\}$ and write $\widecheck\bQ=\widecheck\bQ^{(K(K-1)/2)}\cdots\widecheck\bQ^{(1)}$. In the remainder of this proof, we will use $s(i,j)$ to represent the s-index of the $(i,j)$th rotation matrix $\widecheck\bQ^{ij}$.
	
	For each $(i,j)$, we set the rotation matrix $\widecheck\bQ^{ij}$ as
	\$
	\widecheck\bQ^{ij}_{ii}=\widecheck\bQ^{ij}_{jj}=\cos(\theta^{ij}),\quad \widecheck\bQ^{ij}_{ij}=-\widecheck\bQ^{ij}_{ji}=\sin(\theta^{ij}),\quad \widecheck\bQ^{ij}_{kk}=1,\ \forall k\neq i,j,\quad \widecheck\bQ^{ij}_{kl}=0,\ \textnormal{otherwise},
	\$
	where $\theta^{ij}$ is chosen in a sequential fashion such that the $(i,j)$th element of $\widecheck\bR^{(s(i,j))}\coloneqq(\bR+\bE)\widecheck\bQ^{(1)^\top}\cdots\widecheck\bQ^{(s(i,j))^\top}$ is zero and the diagonal elements of $\widecheck\bR^{(s(i,j))}$ keep positive. Note that $\widecheck\bQ^{ij}$ is by definition an orthogonal matrix. A simple calculation gives that $\theta^{ij}=\theta^{(s(i,j))}=\arctan(\widecheck\bR^{(s(i,j)-1)}_{ij}/\widecheck\bR^{(s(i,j)-1)}_{ii})$. 
	
	Next, we show that $\theta^{ij}$ is a small quantity of order $\cO(\epsilon_0)$ via an deductive argument. First, when $\epsilon_0=\norm{\bE}_{\max}$ is sufficiently small, $\theta^{(1)}=\theta^{12}=\arctan(\bE_{12}/(\bR_{11}+\bE_{11}))$ is a small quantity of order $\cO(\epsilon_0)$. Thus, by definition of $\widecheck\bQ^{12}$, we have $\widecheck\bQ^{(1)}=\widecheck\bQ^{12}=\bI_K+\cO_{\max}(\epsilon_0)$ and
	\$
	\widecheck\bR^{(1)}=(\bR+\bE)(\bI_K+\widecheck\bQ^{(1)^\top}-\bI_K)=\bR+\bE+\bR(\widecheck\bQ^{(1)^\top}-\bI_K)+\cO_{\max}(\epsilon_0^2).
	\$
	Note that the error matrix $\bE^{(1)}\coloneqq\widecheck\bR^{(1)}-\bR$ is again of order $\cO_{\max}(\epsilon_0)$. This implies that $\theta^{(2)}=\theta^{13}=\arctan(\bE^{(1)}_{13}/(\bR_{11}+\bE^{(1)}_{11}))$ is also a small quantity of order $\cO(\epsilon_0)$. Applying this deductive argument $K(K-1)/2$ times, we conclude that all $\theta^{(s)}$, $1\leq s\leq K(K-1)/2$, are small quantities of order $\cO(\epsilon_0)$.
	
	Now we are able to show that $\widecheck\bQ$ has a form of $\bI_K+\widecheck\bP+\cO_{\max}(\epsilon_0^2)$, where $\widecheck\bP\in\RR^{K\times K}$ is a skew-symmetric matrix of order $\cO_{\max}(\epsilon_0)$. Since $\theta^{ij}$ is of order $\cO(\epsilon_0)$, by Taylor expansion, we have $\sin(\theta^{ij})=\theta^{ij}+\cO(\epsilon_0^3)$ and $1-\cos(\theta^{ij})=\cO(\epsilon_0^2)$. Thus, we can rewrite $\widecheck\bQ^{ij}$ as
	\$
	\widecheck\bQ^{ij}=\bI_K+\widecheck\bP^{ij}+\cO_{\max}(\epsilon_0^2),
	\$
	where $\widecheck\bP^{ij}_{ij}=-\widecheck\bP^{ij}_{ij}=\theta^{ij}$ and $\widecheck\bP^{ij}_{kl}=0$ otherwise. Since $\widecheck\bP^{ij}$ is of order $\cO_{\max}(\epsilon_0)$, we have
	\#\label{equ:a1}
	\widecheck\bQ&=\widecheck\bQ^{(K(K-1)/2)}\cdots\widecheck\bQ^{(1)}\notag\\
	&=(\bI_K+\widecheck\bP^{(K(K-1)/2)}+\cO_{\max}(\epsilon_0^2))\cdots(\bI_K+\widecheck\bP^{(1)}+\cO_{\max}(\epsilon_0^2))\notag\\
	&=\bI_K+\sum_{s=1}^{K(K-1)/2}\widecheck\bP^{(s)}+\cO_{\max}(\epsilon_0^2)\notag\\
	&=\bI_K+\widecheck\bP+\cO_{\max}(\epsilon_0^2),
	\#
	where $\widecheck\bP^{(s(i,j))}=\widecheck\bP^{ij}$ and $\widecheck\bP=\sum_{s=1}^{K(K-1)/2}\widecheck\bP^{(s)}$. Since $\widecheck\bP^{(s)}$ is skew-symmetric and of order $\cO_{\max}(\epsilon_0)$ for all $s$, $\widecheck\bP$ is also a skew-symmetric matrix of order $\cO_{\max}(\epsilon_0)$, which concludes the proof of step 1.
	
	\textit{Step 2.} Now we are ready to derive a closed-form expression of $\widecheck\bP$ (maybe up to a higher-order term) by taking upper triangular off-diagonal elements of $(\bR+\bE)\widecheck\bQ^\top$ as zero. Substituting \eqref{equ:a1} into $(\bR+\bE)\widecheck\bQ^\top$, we obtain
	\#\label{equ:a2}
	(\bR+\bE)\widecheck\bQ^\top&=\bR+\bE+\bR\widecheck\bP^\top+\cO_{\max}(\epsilon_0^2)\notag\\
	&=\bR+\bE-\bR\widecheck\bP+\cO_{\max}(\epsilon_0^2),
	\#
	where the second equality follows from the skew-symmetry of $\widecheck\bP$. Since $\bR$ is a lower triangular matrix with positive diagonal elements, $\bR$ is invertible and $\bR^{-1}$ is also a lower triangular matrix. As a result, the matrix $\bR^{-1}(\bR+\bE)\widecheck\bQ^\top$ is also a lower triangular matrix. Multiplying LHS and RHS of \eqref{equ:a2} by $\bR^{-1}$ simultaneously, we obtain
	\#\label{equ:a3}
	\bR^{-1}(\bR+\bE)\widecheck\bQ^\top=\bI_K+\bR^{-1}\bE-\widecheck\bP+\cO_{\max}(\epsilon_0^2).
	\#
	For convenience, we define a function $\cU(\cdot):\RR^{K\times K}\mapsto\RR^{K\times K},\bP\mapsto\cU(\bP)$, where $\cU(\bP)$ takes the upper triangular off-diagonal elements of $\bP$, i.e.,
	\$
	\cU(\bP)_{ij}=\bP_{ij},\ i<j,\quad\cU(\bP)_{ij}=0,\ \textnormal{otherwise}.
	\$
	Since $\bR^{-1}(\bR+\bE)\widecheck\bQ^\top$ is a lower triangular matrix, we have by \eqref{equ:a3} that
	\$
	\cU(\widecheck\bP)=\cU(\bR^{-1}\bE)+\cO_{\max}(\epsilon_0^2).
	\$
	Since $\widecheck\bP$ is skew-symmetric, we get the following closed-form solution of $\widecheck\bP$ (up to a higher-order term),
	\$
	\widecheck\bP=\cU(\bR^{-1}\bE)-(\cU(\bR^{-1}\bE))^\top+\cO_{\max}(\epsilon_0^2).
	\$
	Motivated by the linear expansion of $\widecheck\bP$, we define the following function,
	\$
	f_{\bR}:\RR^{K\times K}\mapsto\RR^{K\times K},\quad\bE\mapsto f_{\bR}(\bE)\coloneqq\cU(\bR^{-1}\bE)-(\cU(\bR^{-1}\bE))^\top.
	\$
	Note that $f_{\bR}$ is linear in the sense that $f_{\bR}(a\bE+b\bF)=af_{\bR}(\bE)+bf_{\bR}(\bF)$ for all $a,b\in\RR$ and $\bE,\bF\in\RR^{K\times K}$. Also, the image $f_{\bR}(\bE)$ is a skew-symmetric matrix, i.e., $(f_{\bR}(\bE))^\top=-f_{\bR}(\bE)$. Moreover, we have $\norm{f_{\bR}(\bE)}_{\rF}\leq \sqrt{2}\norm{\bR^{-1}}_2\norm{\bE}_{\rF}$. By \eqref{equ:a1}, we can rewrite $\widecheck\bQ$ as follows, 
	\#\label{equ:a4}
	\widecheck\bQ=\bI_K+f_{\bR}(\bE)+\cO_{\max}(\epsilon_0^2).
	\#
	Moreover, by definition of $\widecheck\bR$, we have
	\$
	\widecheck\bR=(\bR+\bE)\widecheck\bQ^\top=\bR+\bE-\bR f_{\bR}(\bE)+\cO_{\max}(\epsilon_0^2).
	\$
	
	\textit{Step 3.} In general, when $\bQ\in\cO_{K\times K}$ may differ from $\bI_K$, we can transform the QR decomposition $\widecheck\bR\widecheck\bQ=\bR\bQ+\bE$ suitably and apply the results in the previous two steps to prove the lemma. In specific, we have
	\$
	\widecheck\bR\widecheck\bQ\bQ^\top=\bR+\bE\bQ^\top.
	\$
	When $\epsilon_0=\norm{\bE}_{\max}$ is sufficiently small, $\norm{\bE\bQ^\top}_{\max}\leq\sqrt{K}\epsilon_0$ can also be sufficiently small. In addition, $\widecheck\bQ\bQ^\top$ is still an orthogonal matrix that appears in the QR decomposition of $\bR+\bE\bQ^\top$. Therefore, by \eqref{equ:a4}, we have
	\$
	\widecheck\bQ\bQ^\top=\bI_K+f_{\bR}(\bE\bQ^\top)+\cO_{\max}(K\epsilon_0^2).
	\$
	By multiplying both LHS and RHS of this equation by $\bQ$, we obtain that
	\$
	\widecheck\bQ=\bQ+f_{\bR}(\bE\bQ^\top)\bQ+\cO_{\max}(K^{3/2}\epsilon_0^2).
	\$
	In addition, by definition of $\widecheck\bR$, we have
	\$
	\widecheck\bR&=(\bR\bQ+\bE)\widecheck\bQ^\top\\
	&=(\bR+\bE\bQ^\top)\bQ\widecheck\bQ^\top\\
	&=\bR+\bE\bQ^\top-\bR f_{\bR}(\bE\bQ^\top)+\cO_{\max}(\epsilon_0^2),
	\$
	which concludes the proof.
\end{proof}

\section{Proof of Theorem \ref{thm:3.3}}
\begin{proof}[Proof of Theorem \ref{thm:3.3}]
	First, we use Lemma \ref{lma:3.2} to give a first-order perturbation expansion for the reduced Cholesky factor $\bN^m$ of $\bA^m=(\bN+\bE^m)(\bN+\bE^m)^\top$. Define $\bQ^m\in\cO_{K\times K}$ as an orthogonal matrix such that $\bN^m=(\bN+\bE^m)\bQ^{m\top}$, or equivalently, $(\bR+\bE^{1,m})\bQ^{m\top}$ is a lower triangular matrix with positive diagonal elements. By Lemma \ref{lma:3.2}, when $\norm{\bE^{1,m}}_{\max}\leq\epsilon_0$ is sufficiently small, we have
	\$
	\bQ^{m}=\bI_K+f_{\bR}(\bE^{1,m})+\cO_{\max}(\epsilon_0^2),
	\$
	where $f_{\bR}$ is defined in Lemma \ref{lma:3.2}. By definition of $\bQ^m$, we have
	\$
	\bN^m=(\bN+\bE^m)\bQ^{m\top}=\bN+\bE^m-\bN f_{\bR}(\bE^{1,m})+\cO_{\max}(\epsilon_0^2),
	\$
	where we use the property $f_{\bR}(\bE^{1,m})^\top=-f_{\bR}(\bE^{1,m})$. Using this linear perturbation expansion, we are now able to characterize the Karcher mean $\tilde\bA=\tilde\bN\tilde\bN^\top$ (or $\tilde\bN$) of $\{\bA^m\}_{m=1}^M$ (or $\{\bN^m\}_{m=1}^M$) on the manifold $S^*(p,K)$ (or $\cL^*(p,K)$). By the LRC algorithm, i.e., \eqref{equ:2.3}, we have $\tilde\bN$ is equal to $\frac{1}{M}\sum_{m=1}^M\bN^m$ except that the diagonal elements of $\tilde\bN$ are given by
	\$
	\tilde\bN_{ii}=(\prod_{m=1}^M\bN^m_{ii})^{1/M},\quad\forall 1\leq i\leq K.
	\$ 
	However, when $\epsilon_0$ is sufficiently small, $|\bN^m_{ii}-\bN_{ii}|$ is of order $\cO(\epsilon_0)$ and thus 
	\$
	\tilde\bN_{ii}-\frac{1}{M}\sum_{m=1}^M\bN_{ii}=\cO(\epsilon_0^2),\quad\forall 1\leq i\leq K.
	\$
	Therefore, we have
	\$
	\tilde\bN&=\frac{1}{M}\sum_{m=1}^M\bN^m+\cO_{\max}(\epsilon_0^2)\\
	&=\bN+\frac{1}{M}\sum_{m=1}^M\left(\bE^m-\bN f_{\bR}(\bE^{1,m})\right)+\cO_{\max}(\epsilon_0^2)\\
	&=\bN+\frac{1}{M}\sum_{m=1}^M\bE^m-\bN f_{\bR}(\frac{1}{M}\sum_{m=1}^M\bE^{1,m})+\cO_{\max}(\epsilon_0^2),
	\$
	where the last equality follows from the linear property of $f_{\bR}(\cdot)$.
\end{proof}

\section{Proof of Lemma \ref{lma:4.3}}
\begin{proof}[Proof of Lemma \ref{lma:4.3}]
	Since $\bV,\bLambda$ denote the top $K$ eigenvectors and eigenvalues of $\bSigma$, respectively, we have $\bSigma\bV=\bV\bLambda$ and thus $\bN=\bV\bLambda\bQ^*=\bSigma\bV\bQ^*$. Similarly, we have $\hat\bSigma^m\hat\bV^m=\hat\bV^m\hat\bLambda^m$. Since $\hat\bH^m$ and $\bQ^*$ are both orthogonal matrices, we have
	\$
	(\bN+\hat\bE^m)(\bN+\hat\bE^m)^\top&=(\hat\bV^m\hat\bLambda^m\hat\bH^m\bQ^*)(\hat\bV^m\hat\bLambda^m\hat\bH^m\bQ^*)^\top\\
	&=(\hat\bV^m\hat\bLambda^m)(\hat\bV^m\hat\bLambda^m)^\top\\
	&=\hat\bA^m,
	\$
	which concludes our proof.
\end{proof}

\section{Proof of Lemma \ref{lma:4.4}}

\begin{proof}[Proof of Lemma \ref{lma:4.4}]
	The proof of this lemma is based on a first-order expansion of $\hat\bE^m$. Define $\cE^m=\hat\bSigma^m-\bSigma$ and $\epsilon=\max_m\norm{\cE^m}_2/\Delta_K$. When $\epsilon\leq 1/10$, by Lemma \ref{lma:a2}, we have
	\$
	\norm{\hat\bV^m\hat\bH^m-\bV-g(\cE^m\bV)}_{\rF}\leq 9\sqrt{K}\epsilon^2,
	\$
	where 
	\$
	g:\RR^{p\times K}\mapsto\RR^{p\times K},(\bw_1,\ldots,\bw_K)\mapsto(-\bG_1\bw_1,\ldots,-\bG_K\bw_K),
	\$
	with $\bG_j=\sum_{i>K}(\lambda_i-\lambda_j)^{-1}\bv_i\bv_i^\top$ for $j\in[K]$ and $\lambda_i$/$\bv_i$ being the $i$th eigenvalue/eigenvector of $\bSigma$. By definition of $\hat\bE^m$, we have
	\$
	\hat\bE^m&=((\bSigma + \cE^m)(\bV+(\hat\bV^m\hat\bH^m-\bV))-\bSigma\bV)\bQ^*\\
	&=\cE^m\bV\bQ^*+\bSigma(\hat\bV^m\hat\bH^m-\bV)\bQ^*+\cE^m(\hat\bV^m\hat\bH^m-\bV)\bQ^*\\
	&=\cE^m\bV\bQ^*+\bSigma g(\cE^m\bV)\bQ^*+\cO_{\rF}(\epsilon^2).
	\$
	Since $\textnormal{vec}\circ g\circ \textnormal{vec}^{-1}$ is a linear mapping from $\RR^{pK}$ to $\RR^{pK}$, where $\textnormal{vec}:\RR^{p\times K}\mapsto\RR^{pK}$ is the vectorization mapping, the average of $\hat\bE^m$ can be expressed as
	\$
	\frac{1}{M}\sum_{m=1}^M\hat\bE^m=\frac{1}{M}\sum_{m=1}^M\cE^m\bV\bQ^*+\bSigma g(\frac{1}{M}\sum_{m=1}^M\cE^m\bV)\bQ^*+\cO_{\rF}(\epsilon^2).
	\$
	Thus, by the triangular inequality, we have
	\$
	\norm{\frac{1}{M}\sum_{m=1}^M\hat\bE^m}_{\rF}\leq\norm{\frac{1}{M}\sum_{m=1}^M\cE^m\bV\bQ^*}_{\rF}+\norm{\bSigma g(\frac{1}{M}\sum_{m=1}^M\cE^m\bV)\bQ^*}_{\rF}+\cO(\epsilon^2).
	\$
	Since $\norm{\bQ^*}_2=1$ and $\norm{\bV}_{\rF}=\sqrt{K}$, we have 
	\$
	\norm{\frac{1}{M}\sum_{m=1}^M\cE^m\bV\bQ^*}_{\rF}\leq\sqrt{K}\norm{\frac{1}{M}\sum_{m=1}^M\cE^m}_2.
	\$
	In addition, since $\norm{\bSigma\bG_j}_2\leq \Delta_K^{-1}\lambda_K$ for all $j\in[K]$, we have
	\$
	\norm{\bSigma g(\frac{1}{M}\sum_{m=1}^M\cE^m\bV)\bQ^*}_{\rF}&\leq\Delta_K^{-1}\lambda_K\norm{\frac{1}{M}\sum_{m=1}^M\cE^m\bV}_{\rF}\\
	&\leq\Delta_K^{-1}\lambda_K\sqrt{K}\norm{\frac{1}{M}\sum_{m=1}^M\cE^m}_2.
	\$
	Thus, we have
	\$
	\norm{\frac{1}{M}\sum_{m=1}^M\bE^m}_{\rF}\leq C\norm{\frac{1}{M}\sum_{m=1}^M\cE^m}_2+\cO(\epsilon^2)
	\$
	for some constant $C>0$.
\end{proof}	
\section{Proof of Lemma \ref{lma:4.5}}\label{apx:6}
\begin{proof}[Proof of Lemma \ref{lma:4.5}]
	Similar to Lemma \ref{lma:4.4}, the proof of this lemma is also based on the first-order expansion of $\hat\bE^m$. Define $\cE^m=\hat\bSigma^m-\bSigma$ and $\epsilon=\max_m\norm{\cE^m}_2/\Delta_K$. When $\epsilon\leq 1/10$, by Lemma \ref{lma:a2}, we have
	\$
	\norm{\hat\bV^m\hat\bH^m-\bV-g(\cE^m\bV)}_{\rF}\leq 9\sqrt{K}\epsilon^2,
	\$
	where 
	\$
	g:\RR^{p\times K}\mapsto\RR^{p\times K},(\bw_1,\ldots,\bw_K)\mapsto(-\bG_1\bw_1,\ldots,-\bG_K\bw_K),
	\$
	with $\bG_j=\sum_{i>K}(\lambda_i-\lambda_j)^{-1}\bv_i\bv_i^\top$ for $j\in[K]$ and $\lambda_i$/$\bv_i$ being the $i$th eigenvalue/eigenvector of $\bSigma$. By definition of $\hat\bE^m$, we have
	\$
	\hat\bE^m=\cE^m\bV\bQ^*+\bSigma g(\cE^m\bV)\bQ^*+\cO_{\rF}(\epsilon^2).
	\$
	By the triangular inequality and the fact that $\norm{\cdot}_{\max}\leq\norm{\cdot}_{\rF}$, we have
	\$
	\norm{\hat\bE^m}_{\max}&\leq\norm{\cE^m\bV\bQ^*}_{\max}+\norm{\bSigma g(\cE^m\bV)\bQ^*}_{\max}+\cO(\epsilon^2)\\
	&\leq \sqrt{K}\norm{\cE^m\bV}_{\max}+\sqrt{K}\norm{\bSigma g(\cE^m\bV)}_{\max}+\cO(\epsilon^2),
	\$
	where the second inequality follows from the inequality $\norm{\cdot\bQ^*}_{\max}\leq\sqrt{K}\norm{\cdot}_{\max}$. 
	Thus, to bound $\norm{\hat\bE^m}_{\max}$, it suffices to bound $\norm{\cE^m\bV}_{\max}$, $\norm{\bSigma g(\cE^m\bV)}_{\max}$, and $\epsilon^2=(\max_m\norm{\cE^m}_2/\Delta_K)^2$ separately. To give an upper bound on the first two terms, we will need Proposition 1 in \citet{Pilanci15}, which is presented below for reader's convenience.
	
	\begin{proposition}[Proposition 1 in \citet{Pilanci15}]\label{prop:a1}
		\textit{Let $\cbr{\bz_i}_{i=1}^n\subset\RR^p$ be \text{i.i.d.} samples generated from a zero-mean sub-Gaussian distribution with $\Cov(\bz_i)=\bI_p$. Then there exist some universal constants $C_1,C_2>0$ such that for any subset $\cY\subset\SSS^{p-1}$, we have with probability at least $1-e^{-C_2n\delta^2}$,
			\$
			\sup_{\etab\in\cY}\abr{\etab^{\top}\rbr{\frac{\bZ^{\top}\bZ}{n}-\bI_p}\etab}\leq C_1\frac{\WW(\cY)}{\sqrt{n}}+\delta,
			\$
			where $\bZ^\top=\rbr{\bz_1,\ldots,\bz_n}\in\RR^{p\times n}$ and $\WW(\cY)$ is the Gaussian width of the subset $\cY$.  Specifically, $\WW(\cY)$ is defined by
			\$
			\WW(\cY)=\EE[\sup_{\etab\in\cY}\abr{\inner{\bh}{\etab}}],
			\$
			where the expectation is taken on $\bh\in\RR^p$, which is a standard normal random vector.}	
	\end{proposition}	
	
	\textit{1: bound $\norm{\cE^m\bV}_{\max}$.} Denote by $\be_l\in\RR^p$ the basis vector with value 1 at the $l$th entry and 0 at other entries. Then the max norm has the following expression,
	\$
	\norm{\cE^m\bV}_{\max}=\max_{l\in[p],k\in[K]}|\be_l^\top\cE^m\bv_k|.
	\$
	Let $\bz_i^m=\bSigma^{-1/2}\bx_i^m$ and $\bZ^m=(\bz_1^m,\ldots,\bz_n^m)^\top$. Since $\{\bx_i^m\}_{i=1}^n$ are \textnormal{i.i.d.} sub-Gaussian with mean $\zero$ and covariance $\bSigma$, $\{\bz_i^m\}_{i=1}^n$ are \textnormal{i.i.d.} sub-Gaussian with mean $\zero$ and covariance $\bI_p$.  By definition of $\cE^m$, we have
	\$
	|\be_l^\top\cE^m\bv_k|=|(\bSigma^{1/2}\be_l)^\top(\frac{\bZ^{m\top}\bZ^m}{n}-\bI_p)\bSigma^{1/2}\bv_k|.
	\$
	By the polarization equality, we have
	\$
	(\bSigma^{1/2}\be_l)^\top(\frac{\bZ^{m\top}\bZ^m}{n}-\bI_p)\bSigma^{1/2}\bv_k=\frac{1}{2}\Big\{&(\bSigma^{1/2}\be_l+\bSigma^{1/2}\bv_k)^\top(\frac{\bZ^{m\top}\bZ^m}{n}-\bI_p)(\bSigma^{1/2}\be_l+\bSigma^{1/2}\bv_k)\\
	&-(\bSigma^{1/2}\be_l)^\top(\frac{\bZ^{m\top}\bZ^m}{n}-\bI_p)(\bSigma^{1/2}\be_l)\\
	&-(\bSigma^{1/2}\bv_k)^\top(\frac{\bZ^{m\top}\bZ^m}{n}-\bI_p)(\bSigma^{1/2}\bv_k)\Big\}.
	\$
	Then by the triangular inequality, we have
	\$
	\norm{\cE^m\bV}_{\max}\leq\max_{l\in[p],k\in[K]}\frac{1}{2}\Big\{&|(\bSigma^{1/2}\be_l+\bSigma^{1/2}\bv_k)^\top(\frac{\bZ^{m\top}\bZ^m}{n}-\bI_p)(\bSigma^{1/2}\be_l+\bSigma^{1/2}\bv_k)|\\
	&+|(\bSigma^{1/2}\be_l)^\top(\frac{\bZ^{m\top}\bZ^m}{n}-\bI_p)(\bSigma^{1/2}\be_l)|\\
	&+|(\bSigma^{1/2}\bv_k)^\top(\frac{\bZ^{m\top}\bZ^m}{n}-\bI_p)(\bSigma^{1/2}\bv_k)|\Big\}.
	\$
	Since $\norm{\be_l}_2=\norm{\bv_k}_2=1$, we have
	\$
	\norm{\bSigma^{1/2}(\be_l+\bv_k)}_2\leq 2\norm{\bSigma^{1/2}}_2=2\norm{\bSigma}_2^{1/2},\ \norm{\bSigma^{1/2}\be_l}_2\leq\norm{\bSigma}_2^{1/2},\ \norm{\bSigma^{1/2}\bv_k}_2\leq\norm{\bSigma}_2^{1/2},
	\$
	for all $l\in[p]$ and $k\in[K]$. Define $\cY_1\subset\SSS^{p-1}$ as follows,
	\$
	\cY_1=\left\{\frac{\bSigma^{1/2}(\be_l+\bv_k)}{\norm{\bSigma^{1/2}(\be_l+\bv_k)}_2}\right\}_{l\in[p],k\in[K]}\cup\left\{\frac{\bSigma^{1/2}\be_l}{\norm{\bSigma^{1/2}\be_l}_2}\right\}_{l\in[p]}\cup\left\{\frac{\bSigma^{1/2}\bv_k}{\norm{\bSigma^{1/2}\bv_k}_2}\right\}_{k\in[K]},
	\$
	then we have
	\$
	\norm{\cE^m\bV}_{\max}\leq C_1\cdot\sup_{\etab\in\cY_1}\abr{\etab^{\top}\rbr{\frac{\bZ^{m\top}\bZ^m}{n}-\bI_p}\etab},
	\$
	where $C_1>0$ is a constant dependent on $\norm{\bSigma}_2$. We remark here that in the proof notations $C,C_1,C_2$ represent some universal constants, which may vary according to the context. By Proposition \ref{prop:a1}, there exist some universal constants $C_1,C_2>0$ such that the following inequality holds
	\$
	\norm{\cE^m\bV}_{\max}\leq C_1\frac{\WW(\cY_1)}{\sqrt{n}}+\delta,
	\$
	with probability at least $1-e^{-C_2n\delta^2}$. Since $\cY_1$ is a finite set with cardinality $|\cY_1|\leq Cp$ for some constant $C>0$, by the maximal inequality \citep{Mohri18}, the following inquality
	\$
	\WW(\cY_1)\leq C_1\sqrt{\log(p)}
	\$
	holds for some constant $C_1>0$. Thus with probability at least $1-e^{-C_2n\delta^2}$, we have
	\$
	\norm{\cE^m\bV}_{\max}\leq C_1\sqrt{\frac{\log(p)}{n}}+\delta.
	\$
	
	\textit{2: bound $\norm{\bSigma g(\cE^m\bV)}_{\max}$.} The proof of this step is similar to that of step one. By definition of $g$, we have
	\$
	g(\cE^m\bV)=(-\bG_1\cE^m\bv_1,\ldots,-\bG_K\cE^m\bv_K).
	\$
	Then we may write the max norm as
	\$
	\norm{\bSigma g(\cE^m\bV)}_{\max}=\max_{l\in[p],k\in[K]}|\be_l^\top\bSigma\bG_k\cE^m\bv_k|=\max_{l\in[p],k\in[K]}|\be_l^\top\bSigma\bG_k\bSigma^{1/2}(\frac{\bZ^{m\top}\bZ^m}{n}-\bI_p)\bSigma^{1/2}\bv_k|.
	\$
	Similar to step one, by the polarization equality, the triangular inequality, and the fact $\bG_k=\bG_k^\top$, we have
	\$
	\norm{g(\cE^m\bV)}_{\max}\leq\max_{l\in[p],k\in[K]}\frac{1}{2}\Big\{&|(\bSigma^{1/2}\bG_k\bSigma\be_l+\bSigma^{1/2}\bv_k)^\top(\frac{\bZ^{m\top}\bZ^m}{n}-\bI_p)(\bSigma^{1/2}\bG_k\bSigma\be_l+\bSigma^{1/2}\bv_k)|\\
	&+|(\bSigma^{1/2}\bG_k\bSigma\be_l)^\top(\frac{\bZ^{m\top}\bZ^m}{n}-\bI_p)(\bSigma^{1/2}\bG_k\bSigma\be_l)|\\
	&+|(\bSigma^{1/2}\bv_k)^\top(\frac{\bZ^{m\top}\bZ^m}{n}-\bI_p)(\bSigma^{1/2}\bv_k)|\Big\}.
	\$
	By definition of $\bG_k$, we have
	\$
	\bSigma^{1/2}\bG_k\bSigma=\sum_{i>K}(\lambda_i-\lambda_k)^{-1}\lambda_i^{3/2}\bv_i\bv_i^\top,
	\$
	and thus $\norm{\bSigma^{1/2}\bG_k\bSigma}_{2}\leq\lambda_K^{3/2}\Delta_K^{-1}$. Since $\norm{\be_l}_2=\norm{\bv_k}_2=1$, we have
	\$
	\norm{\bSigma^{1/2}\bG_k\bSigma\be_l+\bSigma^{1/2}\bv_k}_2\leq \norm{\bSigma^{1/2}\bG_k\bSigma}_2+\norm{\bSigma^{1/2}}_2\leq \lambda_K^{3/2}\Delta_K^{-1}+\lambda_1^{1/2},\quad \norm{\bSigma^{1/2}\bG_k\bSigma\be_l}_2\leq\lambda_K^{3/2}\Delta_K^{-1}.
	\$
	Define the following set $\cY_2\subset\SSS^{p-1}$,
	\$
	\cY_2 = \Big\{\frac{\bSigma^{1/2}\bG_k\bSigma\be_l+\bSigma^{1/2}\bv_k}{\norm{\bSigma^{1/2}\bG_k\bSigma\be_l+\bSigma^{1/2}\bv_k}_2}\Big\}_{l\in[p],k\in[K]} \cup \Big\{\frac{\bSigma^{1/2}\bG_k\bSigma\be_l}{\norm{\bSigma^{1/2}\bG_k\bSigma\be_l}_2}\Big\}_{l\in[p],k\in[K]} \cup \Big\{\frac{\bSigma^{1/2}\bv_k}{\norm{\bSigma^{1/2}\bv_k}_2}\Big\}_{k\in[K]}.
	\$
	Then we have
	\$
	\norm{\bSigma g(\cE^m\bV)}_{\max}\leq C_1\cdot \sup_{\etab\in\cY_2}\abr{\etab^{\top}\rbr{\frac{\bZ^{m\top}\bZ^m}{n}-\bI_p}\etab}
	\$
	for some constant $C_1>0$ dependent on $\norm{\bSigma}_2$ and $\Delta_K$. Again by Proposition \ref{prop:a1}, we have with probability at least $1-e^{-C_2n\delta^2}$ that
	\$
	\norm{\bSigma g(\cE^m\bV)}_{\max}\leq C_1\frac{\WW(\cY_2)}{\sqrt{n}}+\delta
	\$
	for some universal constants $C_1,C_2>0$. Since $\cY_2$ is a finite set with cardinality $|\cY_2|\leq Cp$ for some constant $C>0$, by the maximal inequality \citep{Mohri18}, we have
	\$
	\WW(\cY_2)\leq C_1\sqrt{\log(p)}
	\$
	for some constant $C_1>0$. Thus with probability at least $1-e^{-C_2n\delta^2}$, we have
	\$
	\norm{\bSigma g(\cE^m\bV)}_{\max}\leq C_1\sqrt{\frac{\log(p)}{n}}+\delta.
	\$
	
	\textit{3: bound $\epsilon^2$.} We will use the tail bound of $\norm{\cE^m}_2$ in Lemma \ref{lma:a1} to bound $\epsilon^2$. By Lemma \ref{lma:a1}, we have with probability at least $1-e^{-C\sqrt{\frac{\delta n}{r}}}$ that
	\$
	\norm{\cE^m}^2_2\leq \delta
	\$
	for some constant $C>0$	and $r=\textnormal{Tr}(\bSigma)/\lambda_1\leq p$. Then by union bound, we have with probability at least $1-Me^{-C\sqrt{\frac{\delta n}{r}}}$ that
	\$
	\epsilon^2=\max_m\norm{\cE^m}_2^2/\Delta_K^2\leq \delta,
	\$
	for some constant $C>0$.
	
	\textit{Last: bound $\max_m\norm{\hat\bE^m}_{\max}$.} We will combine the results from 1 to 3 and apply a union bound to obtain the upper bound on $\max_m\norm{\hat\bE^m}_{\max}$. In specific, by union bound, we have with probability at least $1-2Me^{-C_1n\delta_1^2}-Me^{-C_2\sqrt{\delta_2n/r}}$ that
	\$
	\max_m\norm{\hat\bE^m}_{\max}\leq C_3\sqrt{\frac{\log(p)}{n}}+\delta_1+\delta_2,
	\$
	for some constants $C_1,C_2,C_3>0$. Take $\delta_1=\sqrt{\frac{\log(2Mp)}{C_1n}}$ and $\delta_2=\frac{\log(Mp)^2r}{C_2^2n}$, then we have with probability at least $1-2p^{-1}$ that,
	\$
	\max_m\norm{\hat\bE^m}_{\max}\leq C_1\sqrt{\frac{\log(pM)}{n}}+C_2\frac{\log^2(pM)r}{n},
	\$
	for some constants $C_1,C_2>0$. When $n\gtrsim \log^3(pM)r^2$, with probability at least $1-2p^{-1}$ the following bound
	\$
	\max_m\norm{\hat\bE^m}_{\max}\leq C\sqrt{\frac{\log(pM)}{n}}
	\$
	holds for some constant $C>0$.
\end{proof}

\section{Proof of Theorem \ref{thm:4.6}}\label{apx:7}
\begin{proof}[Proof of Theorem \ref{thm:4.6}]
	We will combine the results in Theorem \ref{thm:3.3}, Lemma \ref{lma:4.3}, Lemma \ref{lma:4.4}, and Lemma \ref{lma:4.5} to prove this theorem. Recall that $\hat\bE^m=\hat\bSigma^m\hat\bV^m\hat\bH^m\bQ^*-\bSigma\bV\bQ^*$, $\cE^m=\hat\bSigma^m-\bSigma$, $\epsilon=\max_m\norm{\cE^m}_2/\Delta_K$, and $\epsilon_0=\max_m\norm{\hat\bE^m}_{\max}$. In addition, we partition $\bN=(\bR^\top\ \bB^\top)^\top$ and $\hat\bE^m=(\hat\bE^{1,m^\top}\ \hat\bE^{2,m^\top})^\top$ such that $\bR,\hat\bE^{1,m}\in\RR^{K\times K}$ and $\bB,\hat\bE^{2,m}\in\RR^{(p-K)\times K}$. By Lemma \ref{lma:4.3}, we have $\hat\bA^m=(\bN+\hat\bE^m)(\bN+\hat\bE^m)^\top$ for all $m\in[M]$, where $\bN$ is the reduced Cholesky factor of $\bA=\bV\bLambda^2\bV^\top$. By Lemma \ref{lma:4.5}, $\epsilon_0=\max_m\norm{\hat\bE^m}_{\max}$ is sufficiently small with high probability when $n$ is sufficiently large. Thus, we can apply Theorem \ref{thm:3.3} to the LRC-dPCA algorithm to obtain the desired result. In specific, by Theorem \ref{thm:3.3}, we have
	\$
	\tilde\bN=\bN+\frac{1}{M}\sum_{m=1}^M\hat\bE^m-\bN f_{\bR}(\frac{1}{M}\sum_{m=1}^M\hat\bE^{1,m})+\cO_{\max}(\epsilon_0^2),
	\$
	where $f_{\bR}$ is defined in Lemma \ref{lma:3.2}. By the triangular inequality, we have
	\$
	\norm{\tilde\bN-\bN}_{\rF}&\leq\norm{\frac{1}{M}\sum_{m=1}^M\hat\bE^m}_{\rF}+\norm{\bN f_{\bR}(\frac{1}{M}\sum_{m=1}^M\hat\bE^{1,m})}_{\rF}+\cO(\sqrt{pK}\epsilon_0^2)\\
	&\leq\norm{\frac{1}{M}\sum_{m=1}^M\hat\bE^m}_{\rF}+\sqrt{2}\norm{\bN}_2\norm{\bR^{-1}}_2\norm{\frac{1}{M}\sum_{m=1}^M\hat\bE^{1,m}}_{\rF}+\cO(\sqrt{pK}\epsilon_0^2),
	\$
	where the second inequality is due to the property $\norm{f_{\bR}(\cdot)}_{\rF}\leq\sqrt{2}\norm{\bR^{-1}}_2\norm{\cdot}_{\rF}$. By Lemma \ref{lma:4.4}, we have
	\#\label{equ:a5}
	\norm{\tilde\bN-\bN}_{\rF}\leq C\norm{\frac{1}{M}\sum_{m=1}^M\cE^m}_2+\cO(\epsilon^2)+\cO(\sqrt{p}\epsilon_0^2),
	\#
	for some constant $C>0$. 
	
	Next, we give a high probability bound on $\norm{\tilde\bN-\bN}_{\rF}$ in three steps. First, by Lemma \ref{lma:a1}, we have with probability at least $1-e^{-\frac{\delta_1}{C_1\lambda_1\sqrt{r/
				(Mn)}}}$ that
	\$
	\norm{\frac{1}{M}\sum_{m=1}^M\cE^m}_2\leq \delta_1,
	\$
	for some constant $C_1>0$ and $r=\textnormal{Tr}(\bSigma)/\lambda_1(\bSigma)$. Second, by Lemma \ref{lma:a1} and the union bound, we have with probability at least $1-Me^{-\frac{\delta_2}{C_2\lambda_1\sqrt{r/n}}}$ that
	\$
	\epsilon\leq \delta_2/\Delta_K,
	\$
	for some constant $C_2>0$. Third, by Lemma \ref{lma:4.5}, we have with probability at least $1-2Me^{-C_3n\delta_3^2}-Me^{-C_4\sqrt{\delta_4n/r}}$ that
	\$
	\epsilon_0=\max_m\norm{\hat\bE^m}_{\max}\leq C_5\sqrt{\frac{\log(p)}{n}}+\delta_3+\delta_4,
	\$
	for some constants $C_3,C_4,C_5>0$. By the union bound, we combine these three high probability bounds with \eqref{equ:a5} to obtain the desired result. In specific, with probability at least $1-e^{-\frac{\delta_1}{C_1\lambda_1\sqrt{r/(Mn)}}}-Me^{-\frac{\delta_2}{C_2\lambda_1\sqrt{r/n}}}-2Me^{-C_3n\delta_3^2}-Me^{-C_4\sqrt{\delta_4n/r}}$, the following inequality
	\$
	\norm{\tilde\bN-\bN}_{\rF}\leq \cO(\delta_1)+\cO(\delta_2^2)+\cO(\frac{\sqrt{p}\log(p)}{n})+\cO(\sqrt{p}\delta_3^2)+\cO(\sqrt{p}\delta_4^2),
	\$
	holds for some constants $C_1,C_2,C_3,C_4>0$. Take $\delta_1=C_1\lambda_1\log(p)\sqrt{r/(Mn)}$, $\delta_2=C_2\lambda_1\log(pM)\sqrt{r/n}$, $\delta_3=\sqrt{\frac{\log(2pM)}{C_3n}}$, and $\delta_4=\frac{\log^2(pM)r}{C_4^2n}$, then we have with probability at least $1-4p^{-1}$ that
	\$
	\norm{\tilde\bN-\bN}_{\rF}&\leq \cO(\frac{\log(p)\sqrt{r}}{\sqrt{Mn}})+\cO(\frac{\log^2(pM)r}{n})+\cO(\frac{\sqrt{p}\log(pM)}{n})+\cO(\frac{\sqrt{p}\log^4(pM)r^2}{n^2})\\
	&\leq \cO(\frac{\log(p)\sqrt{r}}{\sqrt{Mn}})+\cO(\frac{(\log^2(pM)r)\vee(\log(pM)\sqrt{p})}{n})+\cO(\frac{\sqrt{p}\log^4(pM)r^2}{n^2}).
	\$
	When $n\gtrsim (\log^2(pM)\sqrt{p}r)\vee(\log^3(pM)r^2)$, we have with probability at least $1 - 4p^{-1}$ that
	\$
	\norm{\tilde\bN-\bN}_{\rF}\leq\cO(\frac{\log(p)\sqrt{r}}{\sqrt{Mn}})+\cO(\frac{(\log^2(pM)r)\vee(\log(pM)\sqrt{p})}{n}).
	\$
	In addition, when $n\gtrsim \frac{M((\log^4(pM)r^2)\vee(\log^2(pM)p))}{\log^2(p)r}$, we have with probability at least $1-4p^{-1}$ that
	\$
	\norm{\tilde\bN-\bN}_{\rF}\leq\cO(\frac{\log(p)\sqrt{r}}{\sqrt{Mn}}),
	\$
	which concludes the proof.
\end{proof}	

\section{Auxiliary Lemmas}\label{apx:lma}
The following lemma gives a tail bound of $\norm{\hat\bSigma-\bSigma}_2$ in the sub-Gaussian case.

\begin{lemma}[Lemma 3 in \citet{Fan19}]\label{lma:a1}
	\textit{Suppose $\{\bx_i\}_{i=1}^n\subset\RR^p$ are $\text{i.i.d.}$ sub-Gaussian with mean $\zero$ and covariance $\bSigma$. Let $\hat\bSigma=\frac{1}{n}\sum_{i=1}^n\bx_i\bx_i^\top$ be the sample covariance matrix, $\{\lambda_j\}_{j=1}^p$ be the eigenvalues of $\bSigma$ sorted in descending order, and $r=\textnormal{Tr}(\bSigma)/\lambda_1$. There exist constants $C_1\geq 1$ and $C_2\geq 0$ such that when $n\geq r$, we have
		\$
		\PP(\norm{\hat\bSigma-\bSigma}_2\geq s)\leq \textnormal{exp}\left(-\frac{s}{C_1\lambda_1\sqrt{r/n}}\right),\forall s\geq 0,
		\$
		and $\norm{\norm{\hat\bSigma-\bSigma}_2}_{\psi_1}\leq C_2\lambda_1\sqrt{r/n}$.}	
\end{lemma}

The following lemma provides a first-order expansion of $\hat\bV\hat\bH$ around $\bV$, where $\hat\bH=\argmin_{\bO\in\cO_{K\times K}}\norm{\hat\bV\bO-\bV}_{\rF}$.  

\begin{lemma}[Lemma 8 in \citet{Fan19}]\label{lma:a2}
	\textit{Let $\bSigma,\hat\bSigma\in\RR^{p\times p}$ be symmetric matrices with eigenvalues $\{\lambda_i\}_{i=1}^p$ and $\{\hat\lambda_i\}_{i=1}^p$ (in descending order) and eigenvectors $\{\bv_j\}_{j=1}^p,\{\hat\bv_j\}_{j=1}^p$ such that $\bSigma\bv_j=\lambda_j\bv_j$ and $\hat\bSigma\hat\bv_j=\hat\lambda_j\hat\bv_j$ for $j\in[p]$. Define $\cE=\hat\bSigma-\bSigma$, $S=\{s+1,\ldots,s+K\}$ for some fixed $s\in\{0,1,\ldots,p-K\}$, $\bG_j=\sum_{i\notin S}(\lambda_i-\lambda_{j})^{-1}\bv_i\bv_i^\top$ for $j\in[K]$, and
		\$
		g:\RR^{p\times K}\mapsto\RR^{p\times K},(\bw_1,\ldots,\bw_K)\mapsto(-\bG_1\bw_1,\ldots,-\bG_K\bw_K).
		\$
		Let $\bV=(\bv_{s+1},\ldots,\bv_{s+K})$, $\hat\bV=(\hat\bv_{s+1},\ldots,\hat\bv_{s+K})$, $\bH=\hat\bV^\top\bV$, and $\hat\bH=\textnormal{sgn}(\bH)\coloneqq \bU_1\bU_2^\top$, where $\bH=\bU_1\bGamma\bU_2^\top$ is the unique singular value decomposition of $\bH$. If $\Delta =\min\{\lambda_s-\lambda_{s+1},\lambda_{s+K}-\lambda_{s+K+1}\}>0$ and $\epsilon=\norm{\cE}_2/\Delta\leq 1/10$, where $\lambda_0=\infty$ and $\lambda_{p+1}=-\infty$, we have
		\$
		\norm{\hat\bV\hat\bH-\bV-g(\cE\bV)}_{\rF}\leq 9\epsilon\norm{g(\cE\bV)}_{\rF}.
		\$}
\end{lemma}
It is worth noting that since $\norm{g(\cdot)}_{\rF}\leq\Delta^{-1}\norm{\cdot}_{\rF}$, the Frobenius norm of the remainder term $\hat\bV\hat\bH-\bV-g(\cE\bV)$ is of order 
\$
9\epsilon\norm{g(\cE\bV)}_{\rF}\leq9\epsilon\Delta^{-1}\norm{\cE\bV}_{\rF}\leq9\sqrt{K}\epsilon\Delta^{-1}\norm{\cE\bV}_2\leq9\sqrt{K}\epsilon^2.
\$

\end{document}